\documentclass{article}

\usepackage[preprint, nonatbib]{neurips_2023}

\usepackage[utf8]{inputenc} % allow utf-8 input
\usepackage[T1]{fontenc}    % use 8-bit T1 fonts
\usepackage{hyperref}       % hyperlinks
\usepackage{url}            % simple URL typesetting
\usepackage{booktabs}       % professional-quality tables
\usepackage{amsfonts}       % blackboard math symbols
\usepackage{nicefrac}       % compact symbols for 1/2, etc.
\usepackage{microtype}      % microtypography
\usepackage{xcolor}         % colors
\usepackage{multirow}
\usepackage[toc,page]{appendix}

\usepackage{amsmath, amssymb, amsthm, bbm, mathtools}
\newtheorem{prop}{Proposition}

\usepackage[noend]{algpseudocode}
\usepackage{algorithm}

\usepackage{subfigure}
\usepackage{graphicx}

\title{Learning Representations on the Unit Sphere: Investigating Angular Gaussian and Von Mises-Fisher Distributions for Online Continual Learning}
\author{Nicolas Michel\thanks{This work has received support from Agence Nationale de la Recherche (ANR) for the project APY, with reference ANR-20-CE38-0011-02. This work was granted access to the HPC resources of IDRIS under the allocation 2022-AD011012603 made by GENCI}\\
Univ Gustave Eiffel, CNRS, LIGM \\
F-77454 Marne-la-Vallée, France \\
\texttt{nicolas.michel@esiee.fr}
\And
Giovanni Chierchia \\
Univ Gustave Eiffel, CNRS, LIGM \\
F-77454 Marne-la-Vallée, France \\
\texttt{giovanni.chierchia@esiee.fr}
\And
Romain Negrel \\
Univ Gustave Eiffel, CNRS, LIGM \\
F-77454 Marne-la-Vallée, France \\
\texttt{romain.negrel@esiee.fr}
\And
Jean-François Bercher \\
Univ Gustave Eiffel, CNRS, LIGM \\
F-77454 Marne-la-Vallée, France \\
\texttt{jf.bercher@esiee.fr}
}

\begin{document}

\maketitle

\begin{abstract}
We use the maximum a posteriori estimation principle for learning representations distributed on the unit sphere. We propose to use the angular Gaussian distribution, which corresponds to a Gaussian projected on the unit-sphere and derive the associated loss function. We also consider the von Mises-Fisher distribution, which is the conditional of a Gaussian in the unit-sphere. The learned representations are pushed toward fixed directions, which are the prior means of the Gaussians; allowing for a learning strategy that is resilient to data drift. This makes it suitable for online continual learning, which is the problem of training neural networks on a continuous data stream, where multiple classification tasks are presented sequentially so that data from past tasks are no longer accessible, and data from the current task can be seen only once. 
To address this challenging scenario, we propose a memory-based representation learning technique equipped with our new loss functions. Our approach does not require negative data or knowledge of task boundaries and performs well with smaller batch sizes while being computationally efficient. 
We demonstrate with extensive experiments that the proposed method outperforms the current state-of-the-art methods on both standard evaluation scenarios and realistic scenarios with blurry task boundaries. For reproducibility, we use the same training pipeline for every compared method and share the code at https://github.com/Nicolas1203/ocl-fd.
\end{abstract}

\begin{figure*}[!ht]
    \centering
    \includegraphics[width=0.6\textwidth]{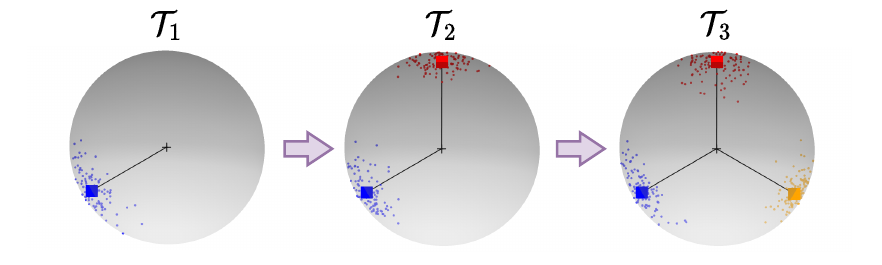}
    \vskip -0.22in
    \caption{Training with fixed directions overview. Each class is assigned to a fixed vector of the standard basis. When changing task $\mathcal{T}$, new classes are encountered and mapped to remaining standard basis vectors. Best viewed in color.}
    \label{fig:overview}
\end{figure*}

\section{Introduction}
% \freefootnote{This work has received support from Agence Nationale de la Recherche (ANR) for the project APY, with reference ANR-20-CE38-0011-02 and was granted access to the HPC resources of IDRIS under the allocation 2022-AD011012603 made by GENCI.}
Deep neural networks can achieve very impressive performances when trained on independent and identically distributed data sampled from a fixed set of classes. In a real-world scenario, however, it may be desirable to train a model on a continuous data stream, where multiple classification tasks are presented sequentially so that the data from the old tasks are no longer accessible when learning new ones and data from the current task can be seen only once. This scenario is known as \emph{online Continual Learning} (CL) and poses a challenge for standard learning algorithms because the distribution of data changes over time (continual setting), and data are not accessible more than once (online setting). If such factors are not adequately taken into account, a trained model may suffer from Catastrophic Forgetting (CF), which is the loss of previously learned knowledge when learning new tasks or from new data. Online CL has seen growing interest in recent years \cite{guo_ocm_2022,aljundi_online_2019,he_online_2021,michel_contrastive_2022,gu_dvc_2022,vedaldi_gdumb_2020,rolnick_experience_2019,mai_online_2021,mai_supervised_2021,lin_pcr_2023}, and several variations have been proposed \cite{hsu_re-evaluating_2019, mai_online_2021}. This paper focuses on class-incremental CL.

Among the different approaches for online CL \cite{mai_online_2021}, memory-based or replay-based methods have shown the best performances for the online setting \cite{rolnick_experience_2019,buzzega_dark_2020,guo_ocm_2022,mai_online_2021, mai_supervised_2021}. In these approaches, a subset of past data is stored while training. When encountering a new batch from the stream, another batch is retrieved from memory and combined with the current batch for training. This mitigates forgetting by seeing past data along with current data. Recently, representation learning techniques combined with replay strategies have shown impressive performances for unsupervised CL \cite{fini_cassle_2022,madaan_lump_2022,davari_probing_2022} and supervised CL \cite{mai_supervised_2021,guo_ocm_2022}. However, contrastive learning-based methods \cite{mai_supervised_2021} often require large batch sizes to benefit from negative samples \cite{gu_dvc_2022}, and distillation-based methods \cite{guo_ocm_2022} require knowledge of task boundaries.

In this work, we propose an algorithm for memory-based representation learning based on a new loss function. Our loss function is devised from the principle of \emph{maximum a posteriori} estimation under the hypothesis that the latent representations are distributed on the unit sphere. Such a hypothesis explicitly takes into account the fact that normalizing the latent vectors is a standard practice in contrastive learning. Specifically, we investigate the angular Gaussian distribution and the Mises-Fisher distribution, both designed for modeling antipodal symmetric directional data.
The peculiarity of the resulting loss function is that the learned representations are pushed toward fixed directions, allowing for a learning strategy that is resilient to data drift and thus suitable for online continual learning. In summary, the contributions of this work are as follows.
\begin{itemize}
    \item We devise a new loss function for representation learning based on the principle of maximum a posteriori estimation. We investigate both the angular Gaussian distribution and the von Mises-Fisher distribution for modeling representations that are restricted on the unit sphere.
    \item The key idea is to essentially assign pre-determined, mutually separated class means in the hidden space (in this case chosen to be the one-hot vectors) and training inputs from each class are coerced to increase overlap with its own class mean. The proposed loss function is resilient to data drift and does not require negative data or knowledge of task boundaries and performs well with smaller batch sizes while being computationally efficient.
    \item We show experimentally on benchmark datasets for online continual learning that the proposed approach outperforms state-of-the-art methods in most scenarios, and is robust to blurry task boundaries.
\end{itemize}

The paper is organized as follows. Section \ref{sec:related} describes related work and formally defines the problem addressed. Section \ref{sec:model_def} explains the mechanisms of our method. Section \ref{sec:exp} presents our experimental results. Section \ref{sec:model_analysis} analyses the behaviour of the proposed approach, and section \ref{sec:conclusion} concludes the paper.

\section{Related Work}
\label{sec:related}
%In this section we describe notable work related to ours.
\paragraph{Representation Learning}
In representation learning, it is common to work with latent vectors projected onto the sphere \cite{chen_simple_2020,mai_supervised_2021, grill_byol_2020, chen_simsiam_2020, wang_understanding_2022, zbontar_barlow_2021}. Contrastive learning is a popular family of approaches for representation learning 
% The main idea is to learn a representation space where similar samples (called positives) should be as close as possible, whereas dissimilar samples (called negatives) should be as far away as possible. 
and has been applied to online CL in previous work, producing state-of-the-art results \cite{mai_online_2021, guo_ocm_2022, cha_co2l_2021,michel_contrastive_2022}. However, contrastive losses require large batch size to sample enough negatives. In this work, we introduce a loss adapted to CL which does not need negative samples.

\paragraph{Class Incremental Learning (CIL)}
One of the most popular continual learning scenario is CIL \cite{hsu_re-evaluating_2019}, which refers to learning from a sequence of tasks, where each task is composed of non-overlapping classes. Formally, consider $\{\mathcal{T}_1,\cdots,\mathcal{T}_K\}$ a learning sequence of $K$ tasks, with $\{\mathcal{D}_1,\cdots,\mathcal{D}_K\}$ the corresponding dataset sequence with $\mathcal{D}_k=(X_k, Y_k)$ the data-label pairs. In CIL, it is assumed that $\forall k,j \in \{1,\cdots,K\}$ if $k\neq j$ then $Y_k\cap Y_{k_2}=\emptyset$ and the number of classes in each task is the same. In this study, we also refer to this setup as \textit{clear} boundaries, meaning that task boundaries are clearly defined as no overlap between tasks exists.

\paragraph{Online Continual Learning}
In online CL, a new constraint is added by restricting the model to seeing the data only once. This problem has been demonstrated to be significantly harder than its offline counterpart and has been the main focus of various recent works \cite{aljundi_online_2019, rolnick_experience_2019,guo_ocm_2022,gu_dvc_2022,buzzega_dark_2020,caccia_new_2022}. Notably, replay-based methods have shown the best performances.
\paragraph{Replay-Based Methods}
In recent years, several methods using fixed memory for replaying past data have addressed online CL. Experience Replay \cite{rolnick_experience_2019} introduces the use of a Reservoir sampling strategy \cite{vitter_random_1985} to replay past data while training on the current task. A-GEM \cite{lopez-paz_gradient_2017} leverages memory data to constrain the current optimization step. DER++ \cite{buzzega_dark_2020} improves ER by adding knowledge distillation between tasks. SCR \cite{mai_supervised_2021} also capitalizes on replaying past data but uses a supervised contrastive loss \cite{khosla_supervised_2020}. OCM \cite{guo_ocm_2022} takes advantage of memory data in online CL with knowledge distillation and maximizes mutual information between previous and current representation with infoNCE \cite{oord_representation_2019}. Likewise, DVC \cite{gu_dvc_2022} combines rehearsal strategies and information maximization. Other strategies using no-memory data usually perform poorly in an online context. In this work, we also focus on online CL and leverage memory data with reservoir sampling. However, we introduce a new loss based on Maximum a Posteriori estimation, defined in section \ref{sec:model_def}.

\paragraph{Fixed Directions}
Recall that classification using cross entropy often ends with selecting components of the logit, which corresponds to a scalar product of the logit with basis vectors. In this sense, the usual practice uses a fixed classifier at the end of the network. This has been made more precise and generalized in previous works~\cite{pernici_incremental_2021,bojanowski2017unsupervised}. However, our theoretical motivations lead to a more general framework from which multiple loss functions can be derived. Additionally, contrary to our approach, proposed losses are computed on unnormalized vectors (not projected on the hypersphere) while in this work, we take into consideration normalized representations and adapt the expression of the proposed loss to the hyperspherical topology. 

\paragraph{Learning on the Unit Sphere}
Hyperspherical loss functions have been proposed in previous studies~\cite{hasnat_von_2017,mettes2019hyperspherical} leveraging von Mises-Fisher distributions. In this work, we propose a more general framework and introduce new loss functions based on Saw distributions~\cite{Saw78}.

\section{Proposed Approach}
\label{sec:model_def}
In this section, we define the proposed approach by introducing new losses for online CL.

\subsection{Representation Learning With Maximum a Posteriori Estimation}
\label{sec:rl_with_map}
%In the following, %we consider a classification problem where 
We are interested in estimating a function $\chi:\mathbb{R}^D \rightarrow [\![1, L]\!]$ that maps an input data to its corresponding class label, with $D \in \mathbb{N}$ the dimensionality of the input data, and $L\in \mathbb{N}$ the number of classes. Let us consider $\textbf{x} \in \mathbb{R}^D$. For a given class $c \in [\![1, L]\!]$, we are interested in the posterior probability
\begin{equation}
\label{eq:bayes}
\mathcal{P}(Y=c|X=\textbf{x})=\frac{\mathcal{P}(X=\textbf{x} | Y=c)\mathcal{P}(Y=c)}{\mathcal{P}(X=\textbf{x})},
\end{equation}
where $X$ and $Y$ are the random variables corresponding to the input and label. 
%As $X$ is a continuous variable, we have
%\begin{equation}
%    \label{eq:pdf1}
%    \mathbb{P}(X \in [x, x + \delta x] | Y=c) \simeq g_c(x) \delta x
%    % \mathbb{P}(X \in [x, x + \delta x] | Y=c) \simeq g_c(x) \delta x \ \text{ ;} \ \ \  \mathbb{P}(X \in [ x, x + \delta x])  \simeq g(x)\delta x
%\end{equation}
%with $g$ the p.d.f.\ of $X$ and $g_c$ the conditional p.d.f.\ of $X$ given $Y=c$. 
In terms of probability densities, we have the posterior density
\begin{equation}
    \label{eq:posterior_X}
    p(Y=c|X=\textbf{x})=\frac{g_c(\textbf{x})\pi_c}{\sum_{\ell=1}^Lg_\ell(\textbf{x})\pi_\ell}
\end{equation}
with $g_c$ the conditional p.d.f.\ of $X$ given $Y=c$ and $\pi_c=\mathcal{P}(Y=c)$ the prior probability for class $c$. 

Let us consider a latent variable $\textbf{z} \in \mathbb{R}^d$ produced by an encoder $\Phi_\theta(.)$ parameterized by $\theta$ such that $\textbf{z}=\Phi_\theta(\textbf{x})$, with $d \in \mathbb{N}$ the dimension of the latent space.
% This representation is the only information used for predicting the label, so we have a Markov chain $X \rightarrow Z \rightarrow Y$ with $\mathcal{P}(Y=c|X=x)=\mathcal{P}(Y=c|Z=z)$.
The posterior from Equation \eqref{eq:posterior_X} can be written according to the random variable $Z$:
\begin{equation}
    \label{eq:p_kz}
    p(Y=c|Z=\textbf{z})=\frac{f_c(\textbf{z})\pi_c}{\sum_{\ell=1}^Lf_\ell(\textbf{z})\pi_\ell}
\end{equation}
where $f_c$ is the conditional distribution of $Z$ given $Y=c$. The objective is now to find the best mapping from $X$ to $Z$ to maximize the posterior distribution. Namely, we aim to find the parameters $\theta^\star$ such that $\theta^\star = \arg\max_{\theta} \, p(Y|Z)$. For a set of $b$ independent observations $(\textbf{z}_i,y_i)_{1\leq i\leq b}$, this amounts to maximizing $p(y_1 \cdots y_b | \textbf{z}_1 \cdots \textbf{z}_b)=\prod_{c=1}^{L}\prod_{i\in I_c}p(Y=c|\textbf{z}_i)$ with $I_c=\{i \in [\![1, b]\!] \mid y_i = c\}$. The posterior distribution in Equation \eqref{eq:p_kz} can be thus expressed as
\begin{equation}
\label{eq:map_prod}
    p(y_1 \cdots y_b | \textbf{z}_1 \cdots \textbf{z}_b) = \prod_{c=1}^L \prod_{i \in I_c} \frac{f_c(\textbf{z}_i)\pi_c}{\sum_{\ell=1}^L f_\ell(\textbf{z}_i)\pi_\ell}.
\end{equation}
Eventually, we express the resulting loss in a batch-by-batch manner. For an incoming batch $\mathcal{B}=(\textbf{x}_i, y_i)_{1\leq i \leq b}$ of size $b$ we minimize Equation \eqref{eq:map_loss} with respect to parameters $\theta$ with $C_{\mathcal{B}}$ the classes in batch $\mathcal{B}$. We take $\pi_l=0$ for classes that are not represented in the current batch and $\pi_l=1$ otherwise.
\begin{equation}
    \label{eq:map_loss}
    \mathcal{L}_{MAP}(\mathcal{B}, \theta) = - \prod_{c \in C_{\mathcal{B}}}\prod_{i \in I_c} \frac{f_c(\Phi_\theta(\textbf{x}_i))\pi_c}{\sum_{\ell \in \mathcal{C}_\mathcal{B}} f_\ell(\Phi_\theta(\textbf{x}_i))\pi_\ell}.
    % -\sum_{c=1}^L\sum_{i \in I_c}\log{\frac{f_c(\Phi_\theta(x_i))\pi_c}{\sum_{n=1}^L f_n(\Phi_\theta(x_i))\pi_n}}
\end{equation}

%Note that the assumption of independent and identically distributed data does not hold in continual learning. We are, however compelled to use it to derive a closed-form expression, well knowing that it is an approximation. Further investigation of this matter is deferred to future work.

\subsection{Saw Distributions on the Unit Sphere}
\label{sec:saw_distrinutions}

The objective defined in Equation \eqref{eq:map_loss} requires us to express the conditional density functions explicitly. In representation learning, it is common to work with normalized vectors in the latent space \cite{chen_simple_2020,mai_supervised_2021, grill_byol_2020, chen_simsiam_2020, wang_understanding_2022, zbontar_barlow_2021}, making it natural to consider distributions on the unit sphere. In a seminal paper \cite{Saw78}, Saw presented a large class of distributions on the sphere parameterized by a mean direction $\boldsymbol{\mu}_c$ (with $\|\boldsymbol{\mu}_c\| = 1$) and a concentration $\kappa\ge0$. These distributions depend on a point $\textbf{z}$ on the sphere (with $\|\textbf{z}\| = 1$)  only through the scalar product $t = \textbf{z}^\top \boldsymbol{\mu}_c$, leading to the general form 
\begin{equation}\label{eq:saw_density}
f_c(\textbf{z}) = a_\kappa \, g_\kappa(\textbf{z}^\top \boldsymbol{\mu}_c).
\end{equation}
%where $\mu_c$ is a mean direction such that $\|\mu_c\| = 1$, $\kappa\ge0$ is the concentration parameter, 
Here above, $a_\kappa$ is a normalization constant, the scalar product corresponds to the cosine similarity and $g_\kappa(t)$ is a non-negative increasing function that must verify a normalizing condition derived from the tangent-normal decomposition of the sphere \cite{Saw78}.
%the points on the unit sphere in the mean direction \cite{Saw78}.
%\begin{equation}
%\frac{1}{B\bigl(\frac12,\frac{d-1}2\bigr)}\int_{-1}^1 g_\kappa(t) (1-t^2)^{(d-3)/2} dt = 1,
%\end{equation}
%with $B$ being the beta function, and $d$ the dimension of the latent space.

A recent study \cite{asao2022convergence} suggests that the representations learned by a neural network have a tendency to follow a Gaussian mixture model, which supports the assumption that the density on the sphere in Equation \eqref{eq:saw_density} shall be derived from such Gaussian mixture. 
In representation learning, we usually project the representation onto the unit sphere, which has a virtue in stabilizing training. Hence, we shall use the probability distribution of a projected  Gaussian distribution onto the unit sphere.

This distribution, in the isotropic case, is the Angular Gaussian (AGD) distribution, which is a Saw distribution \cite{Saw78} (of dimensions $d$) defined by 
\begin{equation}\label{eq:agd_dist}
g_\kappa^{AGD}(t) = 
e^{- \frac{1}{2}\kappa^2} \sum_{n=0}^{\infty} \frac{(\kappa t)^n \, \Gamma\left(\frac{d}{2} + \frac{n}{2}\right) }{n!~\Gamma\left( \frac{d}{2} \right)}.
\end{equation}
This expression of the AGD is given without proof in \cite{Saw78}. We prove it and give different expressions of the probability density for a general Gaussian vector $\mathcal{N}(\mu, \Sigma)$ projected onto the unit-sphere in Appendix.
% \ref{sec:appendixAGD}.
One of these expressions reduces to \eqref{eq:agd_dist} in the isotropic case $\Sigma = \sigma^2 I$, with $\kappa^2 = ||\boldsymbol{\mu}||^2/\sigma^2$, with $||\boldsymbol{\mu}||=||\boldsymbol{z}||=1$. 

Alternatively, starting from a normal distribution with isotropic covariance $\kappa^{-1}{\rm I}$ and mean $\boldsymbol{\mu}_c$, we can condition on $\|\textbf{z}\| = 1$ to obtain the Von Mises–Fisher (vMF) distribution, which is a Saw distribution with %$g(t, \kappa) = \exp(\kappa\,t)$
\begin{equation}\label{eq:vmf_dist}
%f_c^{\rm vMF}(z) = \frac{\exp(\kappa \, \mu_c^T z)}{C_{\rm vMF}(\kappa)}
g_\kappa^{\rm vMF}(t) = \exp(\kappa\,t).
\end{equation}

\subsection{Fixed Directions for Continual Learning}
\label{sec:mu_estimation}
Working with the Saw distributions in Section \ref{sec:saw_distrinutions} requires knowledge of the mean directions of the different classes. With $L$ equiprobable classes, it is natural to use a one-hot encoding of these classes; or, in other words, to assign them to the vertexes of the standard $L$-simplex. 
Alternatively, one could estimate these directions as a parameter of the network \cite{hasnat_von_2017}, or estimate them on the fly. The latter is cumbersome, however, as it requires large batches for the normalized mean estimation to be accurate, and the estimation is strongly biased at the start of each task when new classes are encountered. 
%To overcome these limitations, the mean directions are manually fixed on the unit sphere of $\mathbb{R}^d$.
% We want the distance between each mean to be maximal and to be the same for each mean. 
%A simple solution is to set each direction as a unit vector of the standard basis. 

Formally for a class $c$, we set $\boldsymbol{\mu}_c = \textbf{e}_c = [0, 0, \ldots , 1, 0, \ldots 0]$, a vector where every component is $0$ except the c-th component. With this strategy, the directions have the same distance of $\sqrt{2}$ from one another. Fixing directions also implies that they are independent of batch size or training step, which brings training stability. Such stability is crucial in CL, where new classes can easily conflict with older ones in the latent space \cite{caccia_new_2022}. Finally, fixed directions are obtained at no computational cost. We emphasize that in online CL, maintaining low computational overhead is also an important aspect, as new batches can come in fast succession while storage is limited. An overview of fixing mean direction onto the unit sphere is given in Figure \ref{fig:overview}.

\subsection{Loss Expression}
\label{sec:loss_final}
Taking the logarithm of the objective defined in Equation \eqref{eq:map_loss} we obtained the general loss expression:
% to scenarios where classes are unbalanced, we introduce class weights $\Omega = \{\omega_c \in \mathbb{R}^+  \mid c\in [\![1, L]\!]\}$. Re-weighting classes probability and taking the log leads to the general loss expression for representation learning:
\begin{equation}
\label{eq:map_geolog}
%\mathcal{L}_{\rm MAP}^{\log}(\mathcal{B}, \theta) = -\sum_{c\in C_{\mathcal{B}}} \omega_c\sum_{i \in I_c} \log{\frac{f_c(\Phi_\theta(x_i))\pi_c}{\sum_{\ell\in C_{\mathcal{B}}} f_\ell(\Phi_\theta(x_i))\pi_\ell}}
\mathcal{L}_{\rm MAP}^{\log}(\mathcal{B}, \theta) = -\sum_{c\in C_{\mathcal{B}}} \sum_{i \in I_c} \log{\frac{g_\kappa\big(\textbf{e}_c^\top\Phi_\theta(\textbf{x}_i)\big)\pi_c}{\displaystyle\sum_{\ell\in C_{\mathcal{B}}} g_\kappa\big(\textbf{e}_\ell^\top\Phi_\theta(\textbf{x}_i)\big)\pi_\ell}}
\end{equation}
with $C_{\mathcal{B}}$ the classes in batch $\mathcal{B}$, $\theta$ the parameters of the model, $I_c=\{i \in [\![1, b]\!] \mid y_i = c\}$ the data indexes for class $c$, and $\textbf{e}_c$ the c-th vector of the standard basis. Note that the above general expression includes the losses proposed in \cite{lin_pcr_2023, hasnat_von_2017} based on the von Mises-Fisher distribution.%, we also recover the expression of the loss obtained in \cite{} by computing the cross-entropy with respect to a von Mises-Fisher mixture. 

The Angular Gaussian loss with Fixed Directions (AGD-FD) is obtained by plugging \eqref{eq:agd_dist} into \eqref{eq:map_geolog}, and it is denoted by $\mathcal{L}_{\rm AGD\mbox{-}FD}(\mathcal{B}, \theta)$.  Likewise, the expression of the von-Mises-Fischer loss with Fixed Directions (vMF-FD) is obtained by plugging \eqref{eq:vmf_dist} into \eqref{eq:map_geolog}.  In all our experiments, we work under the assumption that every class has the same prior.

%The expression of the von-Mises-Fischer loss with Fixed Directions (vMF-FD) is obtained by plugging \eqref{eq:vmf_dist} into \eqref{eq:map_geolog}: %, and by simplifying we finally arrive at
%\begin{equation}\label{eq:loss_final}
%\mathcal{L}_{\rm vMF\mbox{-}FD}(\mathcal{B}, \theta) = -\sum_{c \in C_{\mathcal{B}}} \sum_{i \in I_c} \log{\frac{e^{-||\Phi_\theta(x_i) - e_c||^2 / 2v}}{\sum_{\ell \in C_{\mathcal{B}}} e^{-||\Phi_\theta(x_i) - e_\ell||^2 / 2v}}}
%\mathcal{L}_{\rm vMF\mbox{-}FD}(\mathcal{B}, \theta) = -\sum_{c \in C_{\mathcal{B}}} \sum_{i \in I_c} \log{\frac{e^{\kappa \, e_c^\top\Phi_\theta(x_i)}}{\sum_{\ell \in C_{\mathcal{B}}} e^{\kappa \, e_\ell^\top \Phi_\theta(x_i)}}}.
%\mathcal{L}_{\rm vMF\mbox{-}FD}(\mathcal{B}, \theta) = \sum_{c \in C_{\mathcal{B}}} \sum_{i \in I_c} \log\Big( \sum_{\ell \in C_{\mathcal{B}}} e^{\kappa (e_\ell - e_c)^\top \Phi_\theta(x_i)} \Big).
%\end{equation}

%with $C_{\mathcal{B}}$ the classes in batch $\mathcal{B}$ of size $b$, $v$ the variance, $\theta$ the parameters of the model, $I_c=\{i \in [\![1, b]\!] \mid y_i = c\}$ the data indexes for class $c$ and $e_c$ the c-th vector of the standard basis.

\subsection{Implementation Details}

% \vskip -0.3in
\paragraph{Multi-View Batch}
In online CL, the model has to overcome not only a changing data distribution but also the fact that data outside of memory are seen only once. To improve leveraging information from the incoming batch, each image is augmented several times to artificially increase the current batch size and show many 'views' of current data simultaneously. Specifically, for an incoming batch $\mathcal{B}$ and a random augmentation procedure $Aug(.)$, the model is trained on $\mathcal{B}_I=\mathcal{B} \bigcup_{i=1}^n Aug(\mathcal{B})$ with $n$ the number of views. We show in section \ref{sec:nviews_impact} leveraging a multi-view batch helps improve performances.

% \vskip -0.3in
\paragraph{Guillotine Regularization}
Similar to recent Representation Learning techniques, we apply Guillotine Regularization \cite{bordes_guillotine_2022}, to our model. Specifically, we express our model as $\Phi_\theta(.)=(\psi_{\theta_p} \circ \phi_{\theta_r})(.)$ with $\theta = \{\theta_p, \theta_r\}$. $\psi_{\theta_p}$ is referred to as the projection layer and $\phi_{\theta_r}$ the representation layer. The projection layer usually is a simple multilayered perceptron, while the representation layer is a full neural network (e.g. a ResNet). During training, latent variables $\textbf{z}=(\psi_{\theta_p} \circ \phi_{\theta_r})(\textbf{x})$ are used for computing the loss from Equation \eqref{eq:map_geolog}. For inference, the projection layer is dropped, and latent variables $h=\phi_{\theta_r}(\textbf{x})$ are used for the downstream task.

\vskip -0.1in
\subsection{Training Procedure}
Since our proposed approach leads to learning representation, an extra step is needed in order to obtain our final classifier. For fair comparison, we consider that only images stored in memory are available at the end of training. Similar to SCR, when evaluating, the entire memory is used for training an intermediate classifier $\mathcal{C}_w$, with parameters $w$, on top of the frozen representations from $\phi_{\theta_r}$. During the evaluation step $\mathcal{C}_w \circ \phi_{\theta_r}(.)$ is used. A detailed procedure of our method is presented in algorithm \ref{algo:gmap}.

\begin{algorithm}[!ht]
    \begin{algorithmic}\small
    \State\textbf{Input:}\ Data stream $\mathcal{S}$; Memory $\mathcal{M}$; Augmentation procedure $Aug(.)$; Representation Learning Model $\Phi_\theta(.)=(\psi_{\theta_p} \circ \phi_{\theta_r})(.)$; Intermediate classifier $\mathcal{C}_w(.)$; Number of augmentations $n$;
    \State\textbf{Output:}\ End-to-end classifier $\mathcal{C}_w \circ \phi_{\theta_r}(.)$; Memory $\mathcal{M}$;
    \State\texttt{Training Phase:}
        \State $\mathcal{M} \gets \{\}$
        \For{$\mathcal{B}_\mathcal{S} \in \mathcal{S}$}
                \State $\mathcal{B}_\mathcal{M} \gets Retrieve(\mathcal{M})$ \Comment{Random retrieval}
                \State $\mathcal{B}_C \gets \mathcal{B}_\mathcal{S} \cup \mathcal{B}_\mathcal{M}$
                \State $(X_I,Y_I) \gets \mathcal{B}_C \bigcup_{i=1}^n Aug(\mathcal{B}_C)$
                \State $\mathcal{B} \gets \left(\Phi_\theta(X_I), Y_I\right)$
                \State $\theta \gets Adam(\mathcal{L}_{\rm MAP}^{log}(\mathcal{B}, \theta))$ \Comment{Losses from Section \ref{sec:loss_final}}
                \State $\mathcal{M} \gets MemoryUpdate((X_\mathcal{S}, Y_\mathcal{S}), \mathcal{M})$ \Comment{Reservoir Sampling}
        \EndFor
    \State\texttt{Testing Phase:}
    \State $(X_\mathcal{M}, Y_\mathcal{M}) \gets \mathcal{M}$ \Comment{Get all stored memory data}
    \State $H_\mathcal{M} \gets \phi_{\theta_r}(X_\mathcal{M})$\Comment{Encode all memory data}
    \State $w \gets Train(H_\mathcal{M}, Y_\mathcal{M}, \mathcal{C}_w)$ \Comment{Train from frozen representations}
    \State\textbf{return:}\ $\theta_r$;\ $w$;\ $\mathcal{M}$
    \end{algorithmic}
\caption{Proposed Training Method}
\label{algo:gmap}
\end{algorithm}

\vspace{-0.2in}
\section{Experiments}
\label{sec:exp}
% In this section, we describe our experimental setup and analyse obtained results of our method when compared to the current state-of-the-art.
In this section, we describe our experimental setup and analyse obtained results of our method.
\subsection{Towards Real-World Scenarios}
\label{sec:real_world_scenarios}
\begin{figure*}[!ht]
    \centering
    \includegraphics[width=0.6\textwidth]{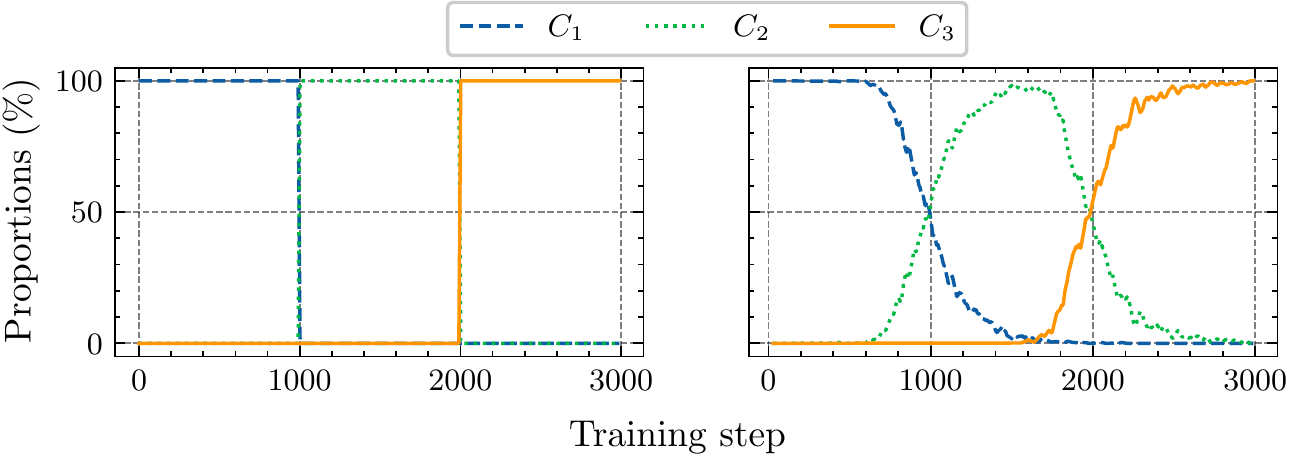}
    \vskip -0.1in
    \caption{Visualisation of class proportions in the incoming batch during training. The left side shows data drift with \textit{clear} boundaries while the right side shows data drift with \textit{blurry} boundaries for $\sigma=1500$ with 3 tasks, 10,000 images per task. $C_i$ corresponds to the classes of task $\mathcal{T}_i$ with $i \in [1,3]$.}
    \label{fig:blurry}
\end{figure*}
%In the following we discuss limitations of describe our evaluation protocol which tends to be more realistic than the standard CIL protocol.

% \paragraph{Long task sequences}
% Online CL tends to be a particularly realistic use case for neural network training. However, the online constraint makes sense only when the data stream is considerably long. Otherwise, storing the entire dataset should be feasible. In this work, we suggest focusing on especially large number of tasks for simulating scenarios where the online constraint is unavoidable. While in early studies, datasets such as MNIST and CIFAR-10 are extensively used, more recent studies use Tiny-ImageNet and other larger-scale datasets, which we believe should be the main experimental focus. In that sense, we focus mainly on longer sequences of tasks and show experimentally that performances on CIFAR-10 are a poor indicator of the performances on harder datasets. However, there is no guarantee that such information will be easily available during training. 

% \noindent\begin{minipage}{.45\textwidth}
\paragraph{Blurry Task Boundaries}
As introduced in section \ref{sec:related}, CIL setups assume clear task boundaries. As an effect, several methods rely on knowing when the task change occurs to use techniques such as distillation \cite{buzzega_dark_2020, guo_ocm_2022}. However, in a real-world scenario, there is no guarantee that task boundaries are clearly defined. Therefore simulating such an environment is crucial for testing models robustness. In that sense, we construct datasets with blurry task boundaries.
% \end{minipage}\hfill
% \begin{minipage}{.53\textwidth}
% \vskip -0.2in
\begin{algorithm}[H]
    \begin{algorithmic}[H]\small
    \State\textbf{Input:}\ Stream sequence with clear boundaries $\mathcal{S}_c$; Scale $\sigma$;
    \State\textbf{Output:}\ Stream sequence with blurry boundaries $\mathcal{S}_b$
    \State $\mathcal{S}_b \leftarrow \{\}$
    \While{$|\mathcal{S}_c| \geq 0$}
        \State $i \sim \mathcal{HN}(\sigma)$\Comment{Sample from a Half-Normal p.d.f.}
        \State $\mathcal{S}_b \gets \mathcal{S}_b \cup \mathcal{S}_c[i]$ \Comment{Add i-th element of $\mathcal{S}_c$ to $\mathcal{S}_b$}
        \State $\mathcal{S}_c \gets \mathcal{S}_c \backslash \{\mathcal{S}_c[i]\}$ \Comment{Drop i-th element of $\mathcal{S}_c$}
    \EndWhile
    \State\textbf{return:}\ $\mathcal{S}_b$
    \end{algorithmic}
\caption{Blurry task boundaries shuffling}
\label{algo:shuffling}
\end{algorithm}
% \end{minipage}

% % \begin{figure}[]
% %     \centering
% %     \subfigure[height=3cm]{\includegraphics[width=0.33\textwidth]{images/notblurry.pdf}} 
% %     \subfigure[height=3cm]{\includegraphics[width=0.4\textwidth]{images/blurry.pdf}}
% %     \caption{(a) Blurry (b) Not blurry \label{fig:blurry}}
% % \end{figure}

% \begin{algorithm}[H]
%     \begin{algorithmic}[H]\small
%     \State\textbf{Input:}\ Stream sequence with clear boundaries $\mathcal{S}_c$; Scale $\sigma$;
%     \State\textbf{Output:}\ Stream sequence with blurry boundaries $\mathcal{S}_b$
%     \State $\mathcal{S}_b \leftarrow \{\}$
%     \While{$|\mathcal{S}_c| \geq 0$}
%         \State $i \sim \mathcal{HN}(\sigma)$\Comment{Sample from a Half-Normal p.d.f.}
%         \State $\mathcal{S}_b \gets \mathcal{S}_b \cup \mathcal{S}_c[i]$ \Comment{Add i-th element of $\mathcal{S}_c$ to $\mathcal{S}_b$}
%         \State $\mathcal{S}_c \gets \mathcal{S}_c \backslash \{\mathcal{S}_c[i]\}$ \Comment{Drop i-th element of $\mathcal{S}_c$}
%     \EndWhile
%     \State\textbf{return:}\ $\mathcal{S}_b$
%     \end{algorithmic}
% \caption{Blurry task boundaries shuffling\label{algo:shuffling}}
% \end{algorithm}

To create such a dataset we start from a dataset with clear boundaries and shuffle it using algorithm \ref{algo:shuffling}. We use a Half-Normal distribution with p.d.f $f_{HN}(y,\sigma)=\frac{\sqrt{2}}{\sigma\sqrt{\pi}}e^{-\frac{y^2}{2\sigma^2}}$ where $y\geq0$ and $\sigma$ is the scale parameter. The resulting data shift can be visualized in Figure \ref{fig:blurry}.

% MINIPAGE
% \noindent\begin{minipage}{.35\textwidth}
% \begin{figure}[H]
%     \centering
%     \includegraphics[\textwidth]{images/blurrynotblurry.pdf}
%     \caption{Blurry and not blurry \label{fig:blurry}}
% \end{figure}
% \end{minipage} \hfill
% \begin{minipage}{.60\textwidth}
% \vskip -0.15in
% \begin{algorithm}[H]
%     \begin{algorithmic}[H]\small
%     \State\textbf{Input:}\ Stream sequence with clear boundaries $\mathcal{S}_c$; Scale $\sigma$;
%     \State\textbf{Output:}\ Stream sequence with blurry boundaries $\mathcal{S}_b$
%     \State $\mathcal{S}_b \leftarrow \{\}$
%     \While{$|\mathcal{S}_c| \geq 0$}
%         \State $i \sim \mathcal{HN}(\sigma)$\Comment{Sample from a Half-Normal p.d.f.}
%         \State $\mathcal{S}_b \gets \mathcal{S}_b \cup \mathcal{S}_c[i]$ \Comment{Add i-th element of $\mathcal{S}_c$ to $\mathcal{S}_b$}
%         \State $\mathcal{S}_c \gets \mathcal{S}_c \backslash \{\mathcal{S}_c[i]\}$ \Comment{Drop i-th element of $\mathcal{S}_c$}
%     \EndWhile
%     \State\textbf{return:}\ $\mathcal{S}_b$
%     \end{algorithmic}
% \caption{Blurry task boundaries shuffling\label{algo:shuffling}}
% \end{algorithm}
% \end{minipage}

% \begin{figure}[]
%     \centering
%     \subfigure[]{\includegraphics[width=0.24\textwidth]{images/sigma0.png}}
%     \subfigure[]{\includegraphics[width=0.24\textwidth]{images/sigma1000.png}} 
%     \subfigure[]{\includegraphics[width=0.24\textwidth]{images/sigma3000.png}} 
%     \caption{(a) Not blurry (b) blurry (c) more blurry\label{fig:blurry}}
% \end{figure}

\paragraph{Random Label Order}
In previous work, experiments are often concluded with the same label order for every run \cite{guo_ocm_2022,buzzega_dark_2020}. However, studies show that label order is important in CL \cite{yoon_scalable_2020}. For fair comparison, we experimented for several runs and for each run, the order of the labels is randomly changed. This ensures more reproducibility and generalization of the proposed results.

\subsection{Evaluation Protocol}
Subsequent is an exhaustive description of the datasets and baselines we considered, as well as details regarding the implementation decisions of our experiment.

\paragraph{Datasets}
To build continual learning environments, variations of standard image classification datasets \cite{krizhevsky_learning_2009,le_tiny_2015} are used. As introduced in section \ref{sec:real_world_scenarios}, we distinguish two variations of original datasets, one with clear task boundaries (\textit{clear} variants) and the other with blurry task boundaries (\textit{blurry} variants). For \textit{clear} variants, each task is composed of non-overlapping classes while in \textit{blurry} variants we introduce some overlap with the procedure previously described. More details are given in Appendix~\cite{michel2023learning}.
% Specifically, we experimented on CIFAR10, CIFAR100 and Tiny ImageNet with blurry and clear task boundaries. \textbf{CIFAR10} contains 50,000 32x32 train images as well as 10,000 test images and is split into 5 tasks containing 2 classes each for a total of 10 distinct classes. \textbf{CIFAR100} contains 50,000 32x32 train images as well as 10,000 test images and is split into 10 tasks containing 10 classes each for a total of 100 distinct classes. \textbf{Tiny ImageNet} is a subset of the ILSVRC-2012 classification dataset and contains 100,000 64x64 train images as well as 10,000 test images and is split into 100 tasks containing 2 classes each for a total of 200 distinct classes.

\vskip -0.1in
\paragraph{Baselines}
Several state-of-the-art approaches for online CL are used for comparison.
% \textbf{offline} is the upper bound. The model is trained without online CL constraints. \textbf{fine-tuned} is the lower bound where the model is trained without strategy to alleviate forgetting.
\textbf{ER} \cite{rolnick_experience_2019} is a memory-based approach using a reservoir sampling \cite{vitter_random_1985} with a cross-entropy loss. \textbf{SCR} \cite{mai_supervised_2021} is a memory-based approach trained using the SupCon loss \cite{khosla_supervised_2020} and a reservoir sampling. \textbf{GDumb} \cite{vedaldi_gdumb_2020} is a method that stores data from stream in memory, ensuring a balanced class selection. The model is trained offline on memory data at inference time. \textbf{AGEM} \cite{chaudhry_efficient_2019} ensures that the average loss of past task does not increase by constraining the gradient using memory data. \textbf{DER++} \cite{buzzega_dark_2020} leverages knowledge distillation and reservoir sampling. \textbf{DVC} \cite{gu_dvc_2022} maximizes information from different images views. \textbf{ER-ACE} \cite{caccia_new_2022} leverages an asymmetric cross-entropy loss along with reservoir sampling. \textbf{OCM} \cite{guo_ocm_2022} maximizes mutual information with infoNCE \cite{oord_representation_2019} and uses reservoir sampling. \textbf{PFC}~\cite{pernici_incremental_2021} which combined experience replay and Pre-Fixed Classifiers.

\begin{table*}[!ht]
    \setlength{\tabcolsep}{4pt}
    \centering
    \resizebox{0.95\textwidth}{!}{\begin{tabular}{l|ll|lll|lll}
    \hline
    \multicolumn{1}{}{} &     \multicolumn{2}{c}{CIFAR10}           &    \multicolumn{3}{c}{CIFAR100}  &             \multicolumn{3}{c}{Tiny ImageNet}     \\
       \hline
    \multicolumn{1}{c}{Method} &         \multicolumn{1}{c}{M=500} &         \multicolumn{1}{c}{M=1k} &        \multicolumn{1}{c}{M=1k} &        \multicolumn{1}{c}{M=2k} &        \multicolumn{1}{c}{M=5k} &        \multicolumn{1}{c}{M=2k} &        \multicolumn{1}{c}{M=5k} &       \multicolumn{1}{c}{M=10k} \\
    \hline\hline
    AGEM    &  16.88{\scriptsize ±1.42} &  16.86{\scriptsize ±1.24} &     4.3{\scriptsize ±0.7} &   4.28{\scriptsize ±0.66} &   4.24{\scriptsize ±0.66} &   0.71{\scriptsize ±0.12} &   0.73{\scriptsize ±0.09} &   0.74{\scriptsize ±0.13} \\
    DER++   &  47.01{\scriptsize ±5.76} &   53.18{\scriptsize ±5.8} &  21.74{\scriptsize ±1.72} &  28.42{\scriptsize ±2.45} &  34.92{\scriptsize ±1.52} &   6.65{\scriptsize ±1.12} &  13.58{\scriptsize ±1.47} &  14.82{\scriptsize ±5.21} \\
    DVC     &   55.8{\scriptsize ±4.67} &  61.42{\scriptsize ±2.68} &  19.79{\scriptsize ±2.63} &   23.19{\scriptsize ±3.6} &  27.43{\scriptsize ±3.26} &   2.45{\scriptsize ±1.27} &   1.72{\scriptsize ±0.74} &   2.22{\scriptsize ±1.39} \\
    ER-ACE   &  54.15{\scriptsize ±1.89} &  61.36{\scriptsize ±1.99} &  27.71{\scriptsize ±0.82} &  32.95{\scriptsize ±1.12} &  39.66{\scriptsize ±1.15} &  15.27{\scriptsize ±1.07} &  22.69{\scriptsize ±1.53} &  27.49{\scriptsize ±1.42} \\
    ER      &  52.51{\scriptsize ±6.27} &  59.02{\scriptsize ±3.27} &   23.04{\scriptsize ±0.9} &  29.65{\scriptsize ±1.33} &  35.52{\scriptsize ±1.43} &   12.49{\scriptsize ±0.5} &  20.55{\scriptsize ±0.96} &  24.06{\scriptsize ±1.01} \\
    GDUMB   &  34.06{\scriptsize ±1.81} &  41.42{\scriptsize ±1.25} &  11.43{\scriptsize ±0.69} &  15.74{\scriptsize ±0.61} &  25.53{\scriptsize ±0.44} &   7.07{\scriptsize ±0.38} &   13.79{\scriptsize ±0.5} &   21.72{\scriptsize ±0.4} \\
    PFC & 57.21{\scriptsize ±0.84} & 62.90{\scriptsize ±0.92} & 24.1{\scriptsize±0.90} & 31.1{\scriptsize±1.60} & 38.6{\scriptsize±0.90} & 11.73{\scriptsize ±0.73} & 19.15{\scriptsize ±2.2} & 23.51{\scriptsize ±2.15} \\
    SCR     &  60.63{\scriptsize ±1.19} &  68.17{\scriptsize ±0.97} &  30.31{\scriptsize ±0.64} &  36.64{\scriptsize ±0.62} &   40.6{\scriptsize ±0.76} &  19.44{\scriptsize ±0.34} &  23.21{\scriptsize ±0.76} &   24.43{\scriptsize ±0.7} \\
    OCM &  \textbf{68.47{\scriptsize ±1.07}} &   \textbf{72.6{\scriptsize ±1.98}} &  29.09{\scriptsize ±1.41} &  36.67{\scriptsize ±1.01} &  42.49{\scriptsize ±1.45} &  19.38{\scriptsize ±0.61} &   27.52{\scriptsize ±0.8} &   32.3{\scriptsize ±1.34} \\
    % OCM {\small (from \cite{guo_ocm_2022})} &  \textbf{70.0{\scriptsize ±1.3}} &  \textbf{77.2{\scriptsize ±0.5}} & {28.1{\scriptsize ±0.3}} & {35.0{\scriptsize ±0.4}} & {42.4{\scriptsize ±0.5}} & {15.7{\scriptsize ±0.2}} & {21.2{\scriptsize ±0.4}} &  {27.0{\scriptsize ±0.3}} \\
\hline
vMF-FD &  60.17{\scriptsize±2.09} &  69.86{\scriptsize±1.02} &  \underline{32.98{\scriptsize±0.83}} &  \underline{41.04{\scriptsize±0.81}} &  \underline{50.39{\scriptsize±0.75}} &  \underline{19.85{\scriptsize±0.68}} &   \underline{28.8{\scriptsize±0.62}} &  \underline{34.21{\scriptsize±0.69}} \\ 
AGD-FD & \underline{61.29{\scriptsize±1.54}} &  \underline{70.06{\scriptsize±1.11}} &  \textbf{33.77{\scriptsize±0.84}} &  \textbf{41.85{\scriptsize±0.85}} &  \textbf{50.54{\scriptsize±0.67}} &  \textbf{20.46{\scriptsize±0.71}} &  \textbf{29.56{\scriptsize±0.68}} &  \textbf{34.77{\scriptsize±0.52}} \\
\hline\hline
\end{tabular}}
\caption{Final average accuracy (\%) for all methods on datasets CIFAR10 split into 5 tasks, CIFAR100 split into 10 tasks, and TinyIN split into 100 tasks for varying memory sizes $M$. Tasks boundaries are \textit{clear} in this setting. Results are computed over 10 runs, and the means and standard deviations are displayed. Best results are in bold. Second are underlined.}
\label{tab:avg_acc_clear}
\end{table*}

% \vskip -0.1in
\paragraph{Metrics} For evaluation, we use the average accuracy across all tasks at the end of training. This is also known as the final average accuracy \cite{kirkpatrick_overcoming_2017, hsu_re-evaluating_2019, mai_online_2021}.

\paragraph{Implementation Details}
For memory-based models, except GDumb, we use random retrieval and reservoir sampling for memory management. DVC, SCR, and OCM employ a two-layer MLP with 512 neurons for intermediate layers (ReLU activation) and 128 neurons for the projection layer. Since our model requires more dimensions than classes, the projection layer output size remains fixed at 512, but additional dimensions can be added as new classes appear. For all methods, we use a full untrained ResNet18. Stream batch size ($|X_S|$) is 10 and memory batch size ($|X_\mathcal{M}|$) is 64 for all methods. We use a Nearest Class Mean (NCM) classifier for intermediary classification. Other intermediate classifiers can be used but we observed little impact on the accuracy.

% \vskip -0.1in
\paragraph{Hyperparameter Search}
We performed a hyper-parameter search on CIFAR100 with a memory size $M=5k$ and 10 tasks. This search includes data augmentation. Best parameters were kept and used for training on every dataset. For fair comparison, we applied this strategy to all approaches, including ours. For OCM we used the parameters from the original paper. Details regarding parameter selection can be found in Appendix~\cite{michel2023learning}. 

% \vskip -0.1in
\paragraph{Data Augmentation} For DER++, ER-ACE and GDumb we use random crop and random horizontal flip as augmentation. For every other method, we use the same data augmentation procedure composed of random crop, random horizontal flip, color jitter, and random grayscale. For our method, we use a number of views $n=5$. OCM comes with additional data augmentation, which we did not change.

% \vskip -0.1in
\paragraph{Adaptation to Blurry Boundaries} For blurry boundaries, we adapted methods that required knowing task boundaries for this setup. Namely, we detected task changes with simple rules. (1) If new classes appear in the stream batch, a new task starts (2) Every task must be at least 100 batches long. This strategy helped to adapt OCM and DER++ but is limited as it can detect more tasks than desired.

% \vskip -0.1in
\subsection{Experimental Results}
% In the following, we discuss the performances obtained by our method.

% \vskip -0.1in
\paragraph{Clear Boundaries} The proposed approach has shown to outperform every considered baseline on CIFAR-100 and Tiny datasets with clear boundaries, as displayed in Table \ref{tab:avg_acc_clear}. This margin becomes even more significant for larger memory sizes on CIFAR-100 up to $8.05\%$ with $M=5k$, which exhibits better scaling with memory size than compared methods. Moreover, our method outperforms every considered baselines except OCM on CIFAR-10. However, experiments on \textit{blurry} variants show evidence that OCM performances highly rely on task boundaries.

% \vskip -0.1in
\paragraph{Blurry Boundaries} As shown in Table \ref{tab:avg_acc_blurry}, our method outperforms every other method in this scenario. Notably, OCM, which requires precise task change for distillation, suffers from a consequent drop in performance in the blurry scenario while our approach gains performance instead. This demonstrates that our method is more suited to realistic scenarios than current state-of-the-art approaches.

\begin{table*}[!hbt]
    \setlength{\tabcolsep}{4pt}
    \centering
    \resizebox{0.95\textwidth}{!}{\begin{tabular}{l|ll|lll|lll}
    \hline
      \multicolumn{1}{}{} &     \multicolumn{2}{c}{CIFAR10}           &    \multicolumn{3}{c}{CIFAR100}  &             \multicolumn{3}{c}{Tiny ImageNet}     \\
       \hline
    \multicolumn{1}{c}{Method} &         \multicolumn{1}{c}{M=500} &         \multicolumn{1}{c}{M=1k} &        \multicolumn{1}{c}{M=1k} &        \multicolumn{1}{c}{M=2k} &        \multicolumn{1}{c}{M=5k} &        \multicolumn{1}{c}{M=2k} &        \multicolumn{1}{c}{M=5k} &       \multicolumn{1}{c}{M=10k} \\
    \hline\hline
AGEM     &  12.62{\scriptsize ±1.92} &  12.28{\scriptsize ±2.44} &   2.36{\scriptsize ±0.33} &   2.51{\scriptsize ±0.27} &   2.42{\scriptsize ±0.33} &   1.15{\scriptsize ±0.24} &    1.18{\scriptsize ±0.25} &  1.16{\scriptsize ±0.3} \\
DER++    &  49.49{\scriptsize ±5.18} &  55.17{\scriptsize ±4.15} &  26.25{\scriptsize ±2.08} &  28.52{\scriptsize ±10.05} & 33.48{\scriptsize ±4.75} &  11.07{\scriptsize ±2.06} &   18.47{\scriptsize ±2.84} & 22.88{\scriptsize ±4.38} \\
DVC      &  58.26{\scriptsize ±2.29} &  62.38{\scriptsize ±2.89} &   24.5{\scriptsize ±2.02} &   26.3{\scriptsize ±5.74} &  33.16{\scriptsize ±2.57} &   11.6{\scriptsize ±2.02} &   17.78{\scriptsize ±2.5} &  18.16{\scriptsize ±4.05} \\
GDUMB    &  34.06{\scriptsize ±1.81} &  41.42{\scriptsize ±1.25} &  11.43{\scriptsize ±0.69} &  15.74{\scriptsize ±0.61} &  25.53{\scriptsize ±0.44} &   7.08{\scriptsize ±0.39} &  13.79{\scriptsize ±0.76} &  22.35{\scriptsize ±0.23} \\
ER       &  54.55{\scriptsize ±2.04} &   61.7{\scriptsize ±2.41} &  23.68{\scriptsize ±0.95} &  29.84{\scriptsize ±1.83} &  36.27{\scriptsize ±1.52} &  11.33{\scriptsize ±2.03} &  19.14{\scriptsize ±1.46} &  24.51{\scriptsize ±1.63} \\
ER-ACE  &   59.93{\scriptsize ±3.12} &   65.3{\scriptsize ±1.52} &   29.54{\scriptsize ±1.0} &  34.73{\scriptsize ±0.71} &  41.16{\scriptsize ±1.57} &  20.85{\scriptsize ±0.85} &  26.79{\scriptsize ±0.97} &  31.6{\scriptsize ±1.01} \\
SCR      &  61.27{\scriptsize ±1.34} &  68.31{\scriptsize ±1.61} &  30.81{\scriptsize ±0.54} &   36.42{\scriptsize ±0.42} &  40.19{\scriptsize ±0.57} &  19.39{\scriptsize ±0.55} &  23.08{\scriptsize ±0.58} &  24.26{\scriptsize ±0.63} \\
OCM &   47.4{\scriptsize ±3.11} &  51.98{\scriptsize ±6.03} &  26.81{\scriptsize ±1.54} & 34.77{\scriptsize ±0.82} &  40.34{\scriptsize ±1.47} &  18.11{\scriptsize ±0.85} &   25.37{\scriptsize ±1.0} &  29.88{\scriptsize ±0.67}  \\
\hline
vMF-FD &  \underline{63.87{\scriptsize ±1.72}} &  \underline{71.04{\scriptsize ±1.29}} &  \underline{34.68{\scriptsize ±0.76}} &    \underline{42.0{\scriptsize ±0.68}} &   \underline{50.71{\scriptsize ±0.60}} &  \underline{21.71{\scriptsize ±0.54}} &  \underline{30.21{\scriptsize ±0.51}} &  \underline{35.16{\scriptsize ±0.49}} \\
AGD-FD & \textbf{64.38{\scriptsize ±2.0}} &  \textbf{71.59{\scriptsize ±0.99} }&  \textbf{36.32{\scriptsize ±0.68}} &  \textbf{43.51{\scriptsize ±0.35}} &  \textbf{50.84{\scriptsize ±0.72}} &  \textbf{23.57{\scriptsize ±0.4}} &  \textbf{31.75{\scriptsize ±0.42}} &  \textbf{35.96{\scriptsize ±0.47}} \\
\hline\hline
    \end{tabular}}
    \caption{Final average accuracy (\%) for all methods on \textit{blurry} variants with CIFAR10 split into 5 tasks, CIFAR100 split into 10 tasks and TinyIN split into 100 tasks for varying memory sizes $M$. Dataset boundaries are blurred with a scale $\sigma=1500$. Results are computed over 10 runs and the means and standard deviations are displayed. Best results are in bold. Second are underlined.}
    \label{tab:avg_acc_blurry}
    \vspace{-0.3cm}
\end{table*}

\section{Model Analysis}
\label{sec:model_analysis}
In subsequent we study the impact of several hyper-parameters for our method. We also apply some of our methods components such as multi-view batch and guillotine regularization to other methods for a fair comparison.

\vskip -0.05in
\paragraph{Impact of not fixing $\mu_c$}
In section \ref{sec:mu_estimation}, we discussed the choice of $\boldsymbol{\mu}_c$, the mean for class $c$. In the following we compare the effect of fixing $\boldsymbol{\mu}_c$ as   (i) $\boldsymbol{\mu}_c=\textbf{e}_c$ where $\textbf{e}_c$ is the c-th vector of the standard basis; 
% (ii) $\mu_c = \hat{\mu_c}$ where $\hat{\mu_c} = \frac{1}{|\mathcal{B}_c|}\sum_{x\in \mathcal{B}_c}\Phi_\theta(x)$ with $\mathcal{B}_c = \{x\in \mathcal{B} | y=c\}$; 
to its estimation as (ii) $\boldsymbol{\mu}_c = \hat{\boldsymbol{\mu}_c} / ||\hat{\boldsymbol{\mu}_c}||$, the \textit{spherical mean}, where $\hat{\mu_c}$ is the arithmetic mean of class $c$ computed with current batch representations. Given that the mean estimation depends on the batch size, we also explore the effect of increasing the number of images retrieved from memory, denoted as $|X_{\mathcal{M}}|$, to facilitate more accurate estimation for larger batch sizes. Results in Table \ref{tab:fixed_vs_mov} demonstrate that fixing the mean values during training leads to a substantial improvement in the overall accuracy, even when using larger batch sizes.

\begin{table*}[!bht]
 % \vskip -0.1in
  \centering
  \setlength{\tabcolsep}{4pt}
  \begin{tabular}{c|cccccc}
    $\mu_c$ $\backslash$ $|X_{\mathcal{M}}|$ &  16  &        32  &         64  &         128 &         256 &         512 \\
    \hline
    % $\hat{\mu_c}$ & 9.32±1.08 &  9.55±0.38 &  14.78±1.42 &  20.88±1.04 &   24.18±1.8 &  25.56±1.43  \\
    $\hat{\mu_c} / ||\hat{\mu_c}||$ & 16.42±0.25 &  19.23±0.5 &  25.54±1.48 &  35.82±0.65 &  44.54±0.82 &  45.79±0.79 \\
    $e_c$     & \textbf{36.91±1.08} &  \textbf{45.42±0.7} &  \textbf{50.36±0.58} &  \textbf{50.98±0.38} &  \textbf{49.98±0.72} &  \textbf{49.62±0.49} \\
  \end{tabular}
    \vskip -0.05in
   \caption{Final AA on CIFAR100 with 10 tasks and M=5k for fixed mean $\boldsymbol{\mu}_c=\textbf{e}_c$ and spherical mean estimates $\boldsymbol{\mu}_c = \hat{\boldsymbol{\mu}_c} / ||\hat{\boldsymbol{\mu}_c}||$,  and various values of $|X_{\mathcal{M}}|$, the number images retrieved from memory. %$\mu_c=e_c$ corresponds to fixing means to standard basis, $\hat{\mu_c}$ is the arithmetic mean computed over current batch. 
  Means and standard deviations from 5 runs are showed.}
    \label{tab:fixed_vs_mov}
\end{table*}

%\vspace{0.1cm} 
\noindent\begin{minipage}{.477\textwidth}
\paragraph{Impact of concentration parameter $\kappa$} Another hyper-parameter to choose for our method is the concentration parameter $\kappa$. On the one hand, the concentration must be low enough for the problem to be feasible but on the other hand, when the concentration falls below a certain value, Gaussians will overlap, which can lead to lower classification accuracy. To illustrate this effect, Figure \ref{fig:kappa} depicts the impact of different values of $\kappa$ on the classification accuracy, where optimal values occur at $\kappa^2=7$ for vMF and $\kappa^2=0.2$ for AGD.
\end{minipage} \hfill
\begin{minipage}{.45\textwidth}
% \vskip -0.1in
\begin{figure}[H]
    \centering
    \includegraphics[width=0.8\textwidth]{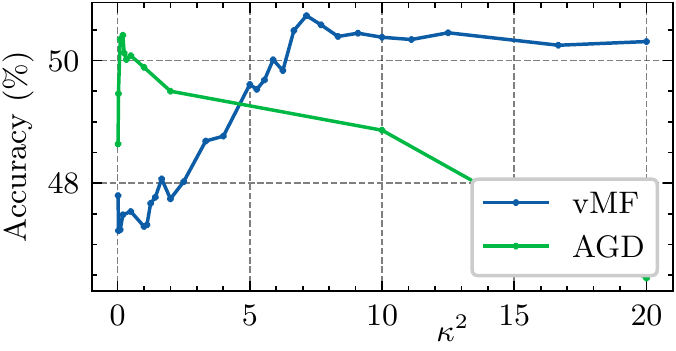}
    \vskip -0.15in
    \caption{Final average accuracy (\%) for $\kappa^2 \in [0.02,20]$ on CIFAR-100 with $M=5k$.}
    \label{fig:kappa}
\end{figure}
\end{minipage}

\begin{table*}[!bht]
 \vskip -0.1in
  \setlength{\tabcolsep}{4pt}
  \centering
  \begin{tabular}{c|cccccccc}
    $n$  & 3 & 4 & 5 & 6 & 7 & 8 & 9 & 10 \\
    \hline
    SCR & 43.9±0.6 &  44.4±0.4 &  44.9±1.0 &  45.5±0.4 &  45.3±0.6 &  45.5±0.8 &  46.5±0.6 &  46.3±0.5 \\
    ER & 42.6±1.0 & 45.0±2.0 & 44.2±1.5 & 44.6±1.1 & 44.3±1.2 & 44.1±1.9 & 43.7±1.0 & 44.3±1.7 \\
    ER-ACE & 42.2±0.9 & 42.1±1.4 & 42.8±0.7 & 42.7±0.8 & 41.8±1.7 & 41.1±1.8 & 40.9±1.3 & - \\
    AGD-FD  & 49.4±0.4 &  \textbf{50.0±0.4} &  50.1±0.7 &  \textbf{51.1±0.4 }&  \textbf{51.1±0.3} &  \textbf{50.9±0.6} &  \textbf{51.7±0.3} &  \textbf{51.2±0.4}  \\
    vMF-FD  & \textbf{49.5±0.6} &  50.0±0.7 &  \textbf{50.7±0.5} &  50.9±0.7 &  50.9±0.4 &  50.4±0.7 &  50.9±0.3 &  50.8±0.3 \\
  \end{tabular}
    \vskip -0.05in
  \caption{Accuracy on CIFAR100 with 10 tasks, $M=5k$ for SCR, ER, ER-ACE, vMF-FD, AGD-FD and the number of views $n \in [3, 10]$. Means and standard deviations over 5 runs are displayed.}
\label{tab:n_augs}
\end{table*} 

\paragraph{Impact of the number of views $n$}
\label{sec:nviews_impact}
Due to the online setting, increasing the number of views has a notable impact on the performances, as shown in Table \ref{tab:n_augs}. To ensure reasonable training time, we use $n=5$ in our experiments, but increasing $n$ could be considered for better performances. To discern the individual impacts of the multi-augmentation and the proposed training loss, we apply the former to SCR and record the resulting performances in Table \ref{tab:n_augs}. Notably, using a multi-view batch also boosts SCR performance, but our method still outperforms SCR even with identical numbers of augmentations. This observation reinforces the efficacy of the proposed loss function.

%\vspace{-0.7cm}
\section{Conclusion}
\label{sec:conclusion}
%\textbf{From chat GPT.} 
%In this paper, we proposed a new loss based on maximum a posteriori estimation with a Gaussian hypothesis, adapted to Continual Learning. We demonstrated through extensive experiments that our method outperforms state-of-the-art methods on standard evaluation scenarios and more realistic datasets with blurry task boundaries. Our approach is computationally efficient and does not require large batch sizes or negative data. As future work, it would be interesting to explore the potential of our method in other domains such as natural language processing, robotics, and reinforcement learning. In summary, the proposed approach offers a promising solution for online continual learning and has the potential to achieve even greater results as it is adapted and refined for different use cases.

This paper proposes a new approach to deal with image classification, adapted to Continual Learning. This approach provides a framework that allows to outperform current state-of-the-art methods. Our learning strategy is to search for the neural network's representations that maximize the a posteriori probability density, a known and consistent approach.

The obtained performance likely results from the choice of data distribution. Since the data used are projected onto the hypersphere, a standard practice in representation learning, we consider distributions on the sphere: the von Mises Fisher distribution and the angular Gaussian distribution. The angular Gaussian distribution, which effectively corresponds to the projection of Gaussian data on the sphere, is the one that shows the best performance, probably because it is the model that most closely approximates the nature of the data. A third ingredient that explains the performance is the use of fixed mean directions, corresponding to each of the classes. This both simplifies the implementation and makes the method robust to data-drift. It is particularly noticeable in the case of blurry boundaries between tasks, where our approach further widens the gap with other approaches. %, especially those that rely on distillation. 
Finally, the results are based on careful implementations and a relevant choice of hyperparameters. Comparisons have been obtained both on standard evaluation scenarios and on more realistic datasets with blurry task boundaries. All methods were carefully re-implemented, often performing better than advertised in the original papers. Beyond pure performance, our approach is computationally efficient and does not require large batch sizes or negative data.

Here we used distributions with an i.i.d.\ assumption on the components. It is possible that the use of a non-diagonal covariance matrix would further improve the results. 
Moreover, other probability distributions, such as a multivariate Student-t on the sphere, can be considered. Finally, the application of our new losses to non-online batch learning should also be considered.

\section*{Acknowledgments}
This work has received support from Agence Nationale de la Recherche (ANR) for the project APY, with reference ANR-20-CE38-0011-02 and was granted access to the HPC resources of IDRIS under the allocation 2022-AD011012603 made by GENCI.

\bibliographystyle{IEEEbib}
\bibliography{refs}

\begin{thebibliography}{10}

\bibitem{guo_ocm_2022}
Yiduo Guo, Bing Liu, and Dongyan Zhao,
\newblock ``Online {Continual} {Learning} through {Mutual} {Information}
  {Maximization},''
\newblock in {\em Proceedings of the 39th {International} {Conference} on
  {Machine} {Learning}}, June 2022, pp. 8109--8126.

\bibitem{aljundi_online_2019}
Rahaf Aljundi, Eugene Belilovsky, Tinne Tuytelaars, Laurent Charlin, Massimo
  Caccia, Min Lin, and Lucas Page-Caccia,
\newblock ``Online {Continual} {Learning} with {Maximal} {Interfered}
  {Retrieval},''
\newblock in {\em Advances in {Neural} {Information} {Processing} {Systems}},
  2019, vol.~32.

\bibitem{he_online_2021}
Jiangpeng He and Fengqing Zhu,
\newblock ``Online continual learning via candidates voting,''
\newblock in {\em Proceedings of the IEEE/CVF Winter Conference on Applications
  of Computer Vision}, 2022, pp. 3154--3163.

\bibitem{michel_contrastive_2022}
Nicolas Michel, Romain Negrel, Giovanni Chierchia, and Jean-Fmn{\c{c}}ois
  Bercher,
\newblock ``Contrastive learning for online semi-supervised general continual
  learning,''
\newblock in {\em 2022 IEEE International Conference on Image Processing
  (ICIP)}. IEEE, 2022, pp. 1896--1900.

\bibitem{gu_dvc_2022}
Yanan Gu, Xu~Yang, Kun Wei, and Cheng Deng,
\newblock ``Not {Just} {Selection}, but {Exploration}: {Online}
  {Class}-{Incremental} {Continual} {Learning} via {Dual} {View}
  {Consistency},''
\newblock in {\em 2022 {IEEE}/{CVF} {Conference} on {Computer} {Vision} and
  {Pattern} {Recognition} ({CVPR})}, June 2022, pp. 7432--7441.

\bibitem{vedaldi_gdumb_2020}
Ameya Prabhu, Philip~HS Torr, and Puneet~K Dokania,
\newblock ``Gdumb: A simple approach that questions our progress in continual
  learning,''
\newblock in {\em Computer Vision--ECCV 2020: 16th European Conference,
  Proceedings, Part II 16}, 2020, pp. 524--540.

\bibitem{rolnick_experience_2019}
David Rolnick, Arun Ahuja, Jonathan Schwarz, Timothy Lillicrap, and Gregory
  Wayne,
\newblock ``Experience {Replay} for {Continual} {Learning},''
\newblock in {\em Advances in {Neural} {Information} {Processing} {Systems}},
  2019, vol.~32.

\bibitem{mai_online_2021}
Zheda Mai, Ruiwen Li, Jihwan Jeong, David Quispe, Hyunwoo Kim, and Scott
  Sanner,
\newblock ``Online continual learning in image classification: An empirical
  survey,''
\newblock {\em Neurocomputing}, vol. 469, pp. 28--51, 2022.

\bibitem{mai_supervised_2021}
Zheda Mai, Ruiwen Li, Hyunwoo Kim, and Scott Sanner,
\newblock ``Supervised contrastive replay: Revisiting the nearest class mean
  classifier in online class-incremental continual learning,''
\newblock in {\em Proceedings of the IEEE/CVF Conference on Computer Vision and
  Pattern Recognition}, 2021, pp. 3589--3599.

\bibitem{lin_pcr_2023}
Huiwei Lin, Baoquan Zhang, Shanshan Feng, Xutao Li, and Yunming Ye,
\newblock ``Pcr: Proxy-based contrastive replay for online class-incremental
  continual learning,''
\newblock in {\em Proceedings of the IEEE/CVF Conference on Computer Vision and
  Pattern Recognition}, 2023, pp. 24246--24255.

\bibitem{hsu_re-evaluating_2019}
Yen-Chang Hsu, Yen-Cheng Liu, Anita Ramasamy, and Zsolt Kira,
\newblock ``Re-evaluating continual learning scenarios: A categorization and
  case for strong baselines,''
\newblock {\em arXiv preprint arXiv:1810.12488}, 2018.

\bibitem{buzzega_dark_2020}
Pietro Buzzega, Matteo Boschini, Angelo Porrello, Davide Abati, and Simone
  Calderara,
\newblock ``Dark experience for general continual learning: a strong, simple
  baseline,''
\newblock in {\em Advances in Neural Information Processing Systems}, 2020,
  vol.~33, pp. 15920--15930.

\bibitem{fini_cassle_2022}
Enrico Fini, Victor G.~Turrisi Da~Costa, Xavier Alameda-Pineda, Elisa Ricci,
  Karteek Alahari, and Julien Mairal,
\newblock ``Self-{Supervised} {Models} are {Continual} {Learners},''
\newblock in {\em 2022 {IEEE}/{CVF} {Conference} on {Computer} {Vision} and
  {Pattern} {Recognition} ({CVPR})}, New Orleans, LA, USA, June 2022, pp.
  9611--9620, IEEE.

\bibitem{madaan_lump_2022}
Divyam Madaan, Jaehong Yoon, Yuanchun Li, Yunxin Liu, and Sung~Ju Hwang,
\newblock ``Representational {Continuity} for {Unsupervised} {Continual}
  {Learning},'' Apr. 2022.

\bibitem{davari_probing_2022}
MohammadReza Davari, Nader Asadi, Sudhir Mudur, Rahaf Aljundi, and Eugene
  Belilovsky,
\newblock ``Probing representation forgetting in supervised and unsupervised
  continual learning,''
\newblock in {\em Proceedings of the IEEE/CVF Conference on Computer Vision and
  Pattern Recognition}, 2022, pp. 16712--16721.

\bibitem{chen_simple_2020}
Ting Chen, Simon Kornblith, Mohammad Norouzi, and Geoffrey Hinton,
\newblock ``A simple framework for contrastive learning of visual
  representations,''
\newblock vol. 119, pp. 1597--1607, 13--18 Jul 2020.

\bibitem{grill_byol_2020}
Jean-Bastien Grill, Florian Strub, Florent Altch{\'e}, Corentin Tallec, Pierre
  Richemond, Elena Buchatskaya, Carl Doersch, Bernardo Avila~Pires, Zhaohan
  Guo, Mohammad Gheshlaghi~Azar, et~al.,
\newblock ``Bootstrap your own latent-a new approach to self-supervised
  learning,''
\newblock {\em Advances in neural information processing systems}, vol. 33, pp.
  21271--21284, 2020.

\bibitem{chen_simsiam_2020}
Xinlei Chen and Kaiming He,
\newblock ``Exploring simple siamese representation learning,''
\newblock in {\em Proceedings of the IEEE/CVF conference on computer vision and
  pattern recognition}, 2021, pp. 15750--15758.

\bibitem{wang_understanding_2022}
Tongzhou Wang and Phillip Isola,
\newblock ``Understanding {Contrastive} {Representation} {Learning} through
  {Alignment} and {Uniformity} on the {Hypersphere},''
\newblock in {\em Proceedings of the 37 th International Conference on Machine
  Learning}. Aug. 2020, PMLR, 119.

\bibitem{zbontar_barlow_2021}
Jure Zbontar, Li~Jing, Ishan Misra, Yann LeCun, and Stéphane Deny,
\newblock ``Barlow {Twins}: {Self}-{Supervised} {Learning} via {Redundancy}
  {Reduction},''
\newblock in {\em Proceedings of the 38th International Conference on Machine
  Learning, ICML 2021}, June 2021, pp. 12310--12320.

\bibitem{cha_co2l_2021}
Hyuntak Cha, Jaeho Lee, and Jinwoo Shin,
\newblock ``Co2l: Contrastive continual learning,''
\newblock {\em Proceedings of the IEEE/CVF International Conference on Computer
  Vision}, pp. 9516--9525, 2021.

\bibitem{caccia_new_2022}
Lucas Caccia, Rahaf Aljundi, Nader Asadi, Tinne Tuytelaars, Joelle Pineau, and
  Eugene Belilovsky,
\newblock ``New insights on reducing abrupt representation change in online
  continual learning,''
\newblock in {\em International Conference on Learning Representations}, 2022.

\bibitem{vitter_random_1985}
Jeffrey~S. Vitter,
\newblock ``Random sampling with a reservoir,''
\newblock {\em ACM Transactions on Mathematical Software}, vol. 11, no. 1, pp.
  37--57, Mar. 1985.

\bibitem{lopez-paz_gradient_2017}
David Lopez-Paz and Marc'Aurelio Ranzato,
\newblock ``Gradient episodic memory for continual learning,''
\newblock {\em Advances in neural information processing systems}, vol. 30,
  2017.

\bibitem{khosla_supervised_2020}
Prannay Khosla, Piotr Teterwak, Chen Wang, Aaron Sarna, Yonglong Tian, Phillip
  Isola, Aaron Maschinot, Ce~Liu, and Dilip Krishnan,
\newblock ``Supervised contrastive learning,''
\newblock {\em Advances in Neural Information Processing Systems}, vol. 33, pp.
  18661--18673, 2020.

\bibitem{oord_representation_2019}
Aaron van~den Oord, Yazhe Li, and Oriol Vinyals,
\newblock ``Representation {Learning} with {Contrastive} {Predictive}
  {Coding},''
\newblock {\em arXiv:1807.03748 [cs, stat]}, Jan. 2019.

\bibitem{pernici_incremental_2021}
Federico Pernici, Matteo Bruni, Claudio Baecchi, Francesco Turchini, and
  Alberto Del~Bimbo,
\newblock ``Class-incremental learning with pre-allocated fixed classifiers,''
\newblock in {\em 2020 25th International Conference on Pattern Recognition
  (ICPR)}. IEEE, 2021, pp. 6259--6266.

\bibitem{bojanowski2017unsupervised}
Piotr Bojanowski and Armand Joulin,
\newblock ``Unsupervised learning by predicting noise,''
\newblock in {\em International Conference on Machine Learning}. PMLR, 2017,
  pp. 517--526.

\bibitem{hasnat_von_2017}
Md~Abul Hasnat, Julien Bohné, Jonathan Milgram, Stéphane Gentric, and Liming
  Chen,
\newblock ``von {Mises}-{Fisher} {Mixture} {Model}-based {Deep} learning:
  {Application} to {Face} {Verification},'' Dec. 2017,
\newblock arXiv:1706.04264 [cs].

\bibitem{mettes2019hyperspherical}
Pascal Mettes, Elise Van~der Pol, and Cees Snoek,
\newblock ``Hyperspherical prototype networks,''
\newblock {\em Advances in neural information processing systems}, vol. 32,
  2019.

\bibitem{Saw78}
John~G. Saw,
\newblock ``A family of distributions on the m-sphere and some hypothesis
  tests,''
\newblock {\em Biometrika}, vol. 65, no. 1, pp. 69--73, 1978.

\bibitem{asao2022convergence}
Yasuhiko Asao, Ryotaro Sakamoto, and Shiro Takagi,
\newblock ``Convergence of neural networks to gaussian mixture distribution,''
\newblock {\em {arXiv}:2204.12100}, 2022.

\bibitem{bordes_guillotine_2022}
Florian Bordes, Randall Balestriero, Quentin Garrido, Adrien Bardes, and Pascal
  Vincent,
\newblock ``Guillotine regularization: Why removing layers is needed to improve
  generalization in self-supervised learning,''
\newblock {\em Transactions on Machine Learning Research}, 2023.

\bibitem{yoon_scalable_2020}
Jaehong Yoon, Saehoon Kim, Eunho Yang, and Sung~Ju Hwang,
\newblock ``Scalable and {Order}-robust {Continual} {Learning} with {Additive}
  {Parameter} {Decomposition},''
\newblock in {\em International Conference on Learning Representation (ICLR)},
  Feb. 2020.

\bibitem{krizhevsky_learning_2009}
Alex Krizhevsky et~al.,
\newblock ``Learning multiple layers of features from tiny images,''
\newblock {\em University of Toronto}, 2009.

\bibitem{le_tiny_2015}
Ya~Le and Xuan Yang,
\newblock ``Tiny imagenet visual recognition challenge,''
\newblock {\em CS 231N}, vol. 7, no. 7, pp. 3, 2015.

\bibitem{michel2023learning}
Nicolas Michel, Giovanni Chierchia, Romain Negrel, and Jean-Fran{\c{c}}ois
  Bercher,
\newblock ``Learning representations on the unit sphere: Application to online
  continual learning,''
\newblock {\em arXiv preprint arXiv:2306.03364}, 2023.

\bibitem{chaudhry_efficient_2019}
Arslan Chaudhry, Marc'Aurelio Ranzato, Marcus Rohrbach, and Mohamed Elhoseiny,
\newblock ``Efficient {Lifelong} {Learning} with {A}-{GEM},''
\newblock {\em arXiv:1812.00420 [cs, stat]}, Jan. 2019.

\bibitem{kirkpatrick_overcoming_2017}
James Kirkpatrick, Razvan Pascanu, Neil Rabinowitz, Joel Veness, Guillaume
  Desjardins, Andrei~A Rusu, Kieran Milan, John Quan, Tiago Ramalho, Agnieszka
  Grabska-Barwinska, et~al.,
\newblock ``Overcoming catastrophic forgetting in neural networks,''
\newblock {\em Proceedings of the national academy of sciences}, vol. 114, no.
  13, pp. 3521--3526, 2017.

\bibitem{mardia2009}
P.E. Jupp and K.V. Mardia,
\newblock {\em Directional Statistics},
\newblock Wiley Series in Probability and Statistics. Wiley, 2009.

\bibitem{PUKKILA1988379}
Tarmo~M. Pukkila and C.~{Radhakrishna Rao},
\newblock ``Pattern recognition based on scale invariant discriminant
  functions,''
\newblock {\em Information Sciences}, vol. 45, no. 3, pp. 379--389, 1988.

\bibitem{paine2018}
P.~J. Paine, S.~P. Preston, M.~Tsagris, and Andrew T.~A. Wood,
\newblock ``An elliptically symmetric angular gaussian distribution,''
\newblock {\em Statistics and Computing}, vol. 28, no. 3, pp. 689--697, 05
  2018.

\bibitem{Saw1973}
John~G. Saw,
\newblock ``Jacobians of singular transformations with applications to
  statistical distribution theory,''
\newblock {\em Communications in Statistics}, vol. 1, no. 1, pp. 81--91, 1973.

\bibitem{Gradshteyn2014}
Daniel Zwillinger, Victor Moll, I.S. Gradshteyn, and I.M. Ryzhik, Eds.,
\newblock {\em Table of Integrals, Series, and Products (Eighth Edition)},
\newblock Academic Press, Boston, eighth edition edition, 2014.

\bibitem{Krizhevsky2009LearningML}
Alex Krizhevsky,
\newblock ``Learning multiple layers of features from tiny images,''
\newblock {\em University of Toronto}, 05 2012.

\bibitem{lomonaco2017core50}
Vincenzo Lomonaco and Davide Maltoni,
\newblock ``Core50: a new dataset and benchmark for continuous object
  recognition,''
\newblock in {\em Conference on robot learning}. PMLR, 2017, pp. 17--26.

\end{thebibliography}

%%%%%%%%%%%%%%%%%%%%%%%%%%%%%%%%%%%%%%%%%%%%%%%%%%%%%%%%%%%%

% \appendix

% \section{Appendix}
\begin{appendices}
\section{Projected-normal or Angular Gaussian Distribution \label{sec:appendixAGD}}

Let $x$ be a random vector of $\mathbb{R}^d$ with a Gaussian distribution of mean $\mu$ and covariance 
matrix $\Sigma$:
\begin{equation}
    f_X(x) = \frac{1}{(2\pi)^\frac{d}{2} |\Sigma|^\frac{1}{2}} \exp{\left(-\frac{1}{2} (x-\mu)^T\Sigma^{-1} (x-\mu)\right)}
\end{equation}
and define by 
\begin{equation}
    u = \frac{x}{||x||} = \frac{x}{(x^Tx)^{\frac{1}{2}}} = \frac{x}{r}
\end{equation}
the projected vector onto the unit sphere $S_{d-1} = \{y \in \mathbb{R}^d: y^Ty = 1 \}$. The marginal of $x$ on $S_{d-1}$ is called \textit{projected-normal} in \cite{mardia2009} or \textit{angular Gaussian} in \cite{PUKKILA1988379}. It seems to be little known and utilized \cite{paine2018}, especially in the case $d>3$. 

We give here several expressions for the density $f_U(u)$ of the normalized vector, recalling the result of \cite{PUKKILA1988379} in terms of a recursively computable integral, proving a result which has been stated in \cite{Saw78} without direct proof, and extending it to the general case. Finally, we provide a closed-form expression in terms of a special function. Let $r = {(x^Tx)^{\frac{1}{2}}}$. The Jacobian of the transformation $x \rightarrow (r, u)$ is $r^{d-1}$ \cite{Saw1973}, so that the density of $(r,u)$ with respect to the surface element $\mathrm{d}\omega_{d-1}$ on the unit sphere, is given by
\begin{align}
    f_{R,U}(r,u) & = \frac{r^{d-1}}{(2\pi)^\frac{d}{2} |\Sigma|^\frac{1}{2}} \exp{\left( -\frac{1}{2}(ru-\mu)^T\Sigma^{-1} (ru-\mu)\right)} \\
    & = \frac{r^{d-1}}{(2\pi)^\frac{d}{2} |\Sigma|} \exp{\left( -\frac{1}{2} \mu^T\Sigma^{-1} \mu\right)} \exp{\left( -\frac{1}{2} r^2 u^T\Sigma^{-1} u + ru^T\Sigma^{-1}\mu \right)}.
\end{align}
The density for $f_U(u)$ is obtained by marginalizing $f_{R,U}(r,u)$ over $r$: $f_U(u) = \int_0^\infty f_{R,U}(r,u) \mathrm{d}r$. Let $r' = r (u^T\Sigma^{-1} u)^\frac{1}{2}$; then
\begin{equation}
    f_U(u) = \frac{(u^T\Sigma^{-1} u)^{-\frac{d}{2}}}{(2\pi)^\frac{d}{2} |\Sigma|^\frac{1}{2}} 
    \exp{\left( -\frac{1}{2} \mu^T\Sigma^{-1} \mu\right)} 
    \int_0^\infty r'^{d-1} \exp{\left( -\frac{1}{2} r'^2 + r' \frac{u^T\Sigma^{-1}\mu}{u^T\Sigma^{-1} u} \right)} \mathrm{d}r'
    \label{eq:fU1}
\end{equation}
Denoting $\lambda = (\mu^T\Sigma^{-1} \mu)^{\frac{1}{2}}$, $\bar{u} = \frac{u}{(u^T\Sigma^{-1} u)^{\frac{1}{2}}}$ and $\bar{\mu} = \frac{\mu}{(\mu^T\Sigma^{-1} \mu)^{\frac{1}{2}}}$, \eqref{eq:fU1} becomes
\begin{equation}
    f_U(u) = \frac{(u^T\Sigma^{-1} u)^{-\frac{d}{2}}}{(2\pi)^\frac{d}{2} |\Sigma|^\frac{1}{2}} 
    \exp{\left( -\frac{1}{2} \lambda^2 \right)} 
    \int_0^\infty r'^{d-1} \exp{\left( -\frac{1}{2} r'^2 + \lambda r' ~ \bar{u}^T\Sigma^{-1}\bar{\mu} \right)} \mathrm{d}r'
    \label{eq:fU2}
\end{equation}
\textbf{Remark}: With $\mu=0$ and $\Sigma = \sigma^2 1$, which means that $x$ is distributed as a centered isotropic Gaussian,  \eqref{eq:fU2} reduces to
\begin{equation}
    f_U(u) = \frac{1}{(2\pi)^\frac{d}{2} } 
    \int_0^\infty r'^{d-1} \exp{\left( -\frac{1}{2} r'^2 \right)} \mathrm{d}r'\\
     = \frac{\Gamma\left(\frac{d}{2}\right)}{2 \pi^\frac{d}{2} } = \frac{1}{\omega_{d-1}}
     \label{eq:unifsphere}
\end{equation}
where we used $u^Tu=1$ and the known property 
\begin{equation}
\int_0^\infty r^{d-1} \exp{\left( -\frac{1}{2} r^2 \right)} \mathrm{d}r = 2^{\frac{d}{2}-1}\Gamma\left(\frac{d}{2}\right).
\label{eq:gammaId}
\end{equation}
Equation \eqref{eq:unifsphere} shows that $f_U(u)$ is the uniform distribution on the unit-sphere, where $\omega_{d-1}$ is the surface of the unit-sphere. 

Starting with \eqref{eq:fU1}, we can now state the first result, which is due to \cite{PUKKILA1988379}.  

\begin{prop}
 With  $\lambda = (\mu^T\Sigma^{-1} \mu)^{\frac{1}{2}}$ and $\alpha = \frac{u^T\Sigma^{-1}\mu}{u^T\Sigma^{-1} u}$, the probability density of the normalized Gaussian vector is
\begin{equation}
       f_U(u) = \frac{(u^T\Sigma^{-1} u)^{-\frac{d}{2}}}{(2\pi)^{\frac{d}{2}-1} |\Sigma|^\frac{1}{2}} 
    \exp{\left( -\frac{1}{2} \left(\lambda^2-\alpha^2\right) \right)} 
    I_d(\alpha) 
    \label{eq:fU_Rao}
\end{equation}
with
\begin{equation}
I_d(\alpha) = \frac{1}{\sqrt{2\pi}} \int_0^\infty r^{d-1} \exp{\left( -\frac{1}{2} (r-\alpha)^2 \right)} \mathrm{d}r    
\label{eq:Id_def}
\end{equation}
and can be computed as 
$$I_d(\alpha) =  \alpha I_{d-1}(\alpha) + (d-2)I_{d-2}(\alpha),$$ 
with $I_1 = \Phi(\alpha)$ and $I_2 = \phi(\alpha) + \alpha \Phi(\alpha)$, where $\phi(.)$ and $\Phi(.)$ are respectively the standard normal probability density function and cumulative distribution function.
\end{prop}
\begin{proof}
    Completing the square in the argument of the exponential under the integral in \eqref{eq:fU1} gives \eqref{eq:fU_Rao}, with the definition of $I_d$ in \eqref{eq:Id_def}. Integration by part of $I_d$ yields the recurrence equation. Finally, the initial values follow by direct calculation.
\end{proof}

The downside of \eqref{eq:fU_Rao} is of course that it depends on an integral form, even if this integral can be easily evaluated by recurrence. From \eqref{eq:fU2}, it is possible to obtain the density as a series. We give here this result in the general case and recover the result stated in \cite{Saw78} without direct proof. 
\begin{prop}
     With  $\lambda = (\mu^T\Sigma^{-1} \mu)^{\frac{1}{2}}$, $\bar{u} = \frac{u}{(u^T\Sigma^{-1} u)^{\frac{1}{2}}}$ and $\bar{\mu} = \frac{\mu}{(\mu^T\Sigma^{-1} \mu)^{\frac{1}{2}}}$, the probability density of the normalized Gaussian vector is
     \begin{equation}
          f_U(u) = \frac{\Gamma\left( \frac{d}{2}\right)}{2\pi^\frac{d}{2}} \, \frac{(u^T\Sigma^{-1} u)^{-\frac{d}{2}}}{|\Sigma|^\frac{1}{2}} 
   e^{ -\frac{1}{2} \lambda^2 }   \sum_{k=0}^\infty  \left( \lambda\bar{u}^T\Sigma^{-1}\bar{\mu}\right)^k \frac{\Gamma\left( \frac{d+k}{2}\right)}{k! \, \Gamma\left( \frac{d}{2}\right)}  
    \label{eq:serisSawgGeneral}
     \end{equation}
\end{prop}
\begin{proof}
    In the integral in \eqref{eq:fU2}, we can expand the exponential $\exp{\left( \lambda r ~ \bar{u}^T\Sigma^{-1}\bar{\mu} \right)}$ in Taylor series, so that
    \begin{align}
    & \int_0^\infty r^{d-1} \exp{\left( -\frac{1}{2} r^2 + \lambda r ~ \bar{u}^T\Sigma^{-1}\bar{\mu} \right)} \mathrm{d}r \\
    & =  \int_0^\infty r^{d-1} \exp{\left( -\frac{1}{2} r^2 \right)} 
    \sum_{k=0}^\infty \frac{1}{k!} \left( \lambda r ~ \bar{u}^T\Sigma^{-1}\bar{\mu}\right)^k \mathrm{d}r \\
    & = \sum_{k=0}^\infty \frac{1}{k!} \left( \lambda\bar{u}^T\Sigma^{-1}\bar{\mu}\right)^k \int_0^\infty r^{d-1+k} \exp{\left( -\frac{1}{2} r^2 \right)} \\
    & = 2^{\frac{d}{2}-1} \sum_{k=0}^\infty \frac{1}{k!} \left( \lambda\bar{u}^T\Sigma^{-1}\bar{\mu}\right)^k \Gamma\left( \frac{d+k}{2}\right)
    \end{align}
    where the last line follows from the identity \eqref{eq:gammaId}. Plugging this in \eqref{eq:fU2} and simplifying yield \eqref{eq:serisSawgGeneral}.
\end{proof}
Note that the first term in \eqref{eq:serisSawgGeneral} is the inverse of the unit-sphere's surface $\omega_{d-1}$. In the isotropic case, that is $\Sigma = \sigma^2$1, \eqref{eq:serisSawgGeneral} reduces to
     \begin{equation}
          f_U(u) = \frac{\Gamma\left( \frac{d}{2}\right)}{2\pi^\frac{d}{2}} \,  
   e^{ -\frac{1}{2} \lambda^2 }   \sum_{k=0}^\infty  \left( \lambda u^T\bar{\mu}\right)^k \frac{\Gamma\left( \frac{d+k}{2}\right)}{k! \, \Gamma\left( \frac{d}{2}\right)}  
    \label{eq:serisSawgIso}
     \end{equation}
     where we used the fact that $u^Tu=1$ and where $\bar{\mu}$ is now $\bar{\mu} = \frac{\mu}{\left(\mu^T\mu\right)^{\frac{1}{2}}}$. This is the formula given in \cite{Saw78}, up to minor notations differences. Finally, for $\mu = 0$, \eqref{eq:serisSawgIso} reduces to the uniform distribution on the unit-sphere  $f_U(u) = 1/\omega_{d-1}$. 

Finally, it is possible to obtain a closed form in terms of a special function. 
\begin{prop}
     With  $\lambda = (\mu^T\Sigma^{-1} \mu)^{\frac{1}{2}}$ and $\gamma=\frac{{u}^T\Sigma^{-1}{\mu}}{(u^T\Sigma^{-1} u)^{\frac{1}{2}}}$, the probability density of the normalized Gaussian vector is
     \begin{equation}
         f_U(u) = \frac{(u^T\Sigma^{-1} u)^{-\frac{d}{2}}}{(2\pi)^\frac{d}{2} |\Sigma|^\frac{1}{2}} 
     e^{ -\frac{1}{2} \lambda^2-\frac{1}{8}\gamma^2} \Gamma(d) D_{-d}\left( {\sqrt{2}\gamma}\right),
    \label{eq:fU_parabolicCylinder}
     \end{equation}
     where $D_{-d}$ is a Parabolic cylinder function.
\end{prop}
\begin{proof}
A result in the celebrated Tables of integrals, Series and Products of Gradshteyn and Ryzhik states, \cite[eq. 3.462]{Gradshteyn2014}, that
\begin{equation}
    \int_0^\infty x^{\nu -1} e^{-\beta x^2 - \gamma x} \mathrm{d}x = (2\beta)^{-\nu/2}\Gamma(\nu) e^{-\frac{\gamma^2}{8\beta}} D_{-\nu}\left( \frac{\gamma}{\sqrt{2\beta}}\right) \text{for } \beta>0, \nu >0
\end{equation}
where $D_{\nu}$ is a parabolic cylinder function, \cite[eq. 9.240]{Gradshteyn2014}. We see that the integral in \eqref{eq:fU2} has precisely this form, with $\nu=d$, $\beta=1/2$, and $\gamma=\lambda \bar{u}^T\Sigma^{-1}\bar{\mu}$. Plugging this in \eqref{eq:fU2} and rearranging yield \eqref{eq:fU_parabolicCylinder}.  
\end{proof}

\section{Hyperparameter Search}

This appendix describes which hyper-parameter values have been tested for every compared method. As described in the main paper, the search has been conducted on one setup, namely CIFAR-100 with M=5k, and the resulting hyper-parameters have been used for all remaining scenarios. This strategy has been applied to every compared method for fair comparison.

\subsection{Augmentation strategy}

Some methods presented prove to gain from simple augmentations rather than more complex augmentations. To obtain the best performances possible for every compared method, we considered two augmentations strategies, which we named \textit{partial} and \textit{full}, respectively.

\paragraph{\textit{Partial} augmentation strategy.} The partial augmentation strategy is, as the name implies, composed of only a subpart of the augmentations used in the \textit{full} strategy. Precisely, it is only a sequence of a random crop and a random horizontal flip, with $p=0.5$.

\paragraph{\textit{Full} augmentation strategy.} The full augmentation strategy is composed a more augmentations. Namely, it is a sequence of random crop, horizontal flip, color jitter and random gray scale. Color jitter parameters are set to $(0.4, 0.4, 0.4, 0.1)$ and $p=0.8$. The probability of applying random gray scale is set to $0.2$.

\subsection{Hyper-parameters table}

Table \ref{tab:all_hp} is an exhaustive list of the hyper-parameters values tried during grid search. Note that for OCM we used the values given in their original work. Similarly, for GDumb, we experimented only with the augmentations and kept the other parameters as given in their original work.

\begin{table}[!ht]
    \setlength{\tabcolsep}{4pt}
    \centering
    % \resizebox{\textwidth}{!}{\begin{tabular}{c|c}
    \begin{tabular}{c|cc}
    Method & Parameter & Values \\
    \hline
    \multirow{5}*{ER}    & optim         & [SGD, Adam] \\
                        & weight decay  & [0, 1e-4] \\
                        & lr            & [0.0001, 0.001, 0.01, 0.1] \\
                        & momentum      & [0, 0.9] \\
                        & aug. strat.   & [full, partial] \\
     \hline
    \multirow{5}*{ER-ACE}   & optim         & [SGD, Adam] \\
                        & weight decay  & [0, 1e-4] \\
                        & lr            & [0.0001, 0.001, 0.01, 0.1] \\
                        & momentum      & [0, 0.9] \\
                        & aug. strat.   & [full, partial] \\
    \hline
    \multirow{5}*{A-GEM}    & optim         & [SGD, Adam] \\
                        & weight decay  & [0, 1e-4] \\
                        & lr            & [0.0001, 0.001, 0.01, 0.1] \\
                        & momentum      & [0, 0.9] \\
                        & aug. strat.   & [full, partial] \\
    \hline
    \multirow{7}*{DER++}    & optim         & [SGD, Adam] \\
                        & weight decay  & [0, 1e-4] \\
                        & lr            & [0.0001, 0.001, 0.01, 0.03] \\
                        & momentum      & [0, 0.9] \\
                        & aug. strat.   & [full, partial] \\
                        & alpha         & [0.1, 0.2, 0.5, 1.0] \\
                        & beta          & [0.5, 1.0] \\
    \hline
    \multirow{5}*{DVC}      & optim         & [SGD, Adam] \\
                        & weight decay  & [0, 1e-4] \\
                        & lr            & [0.0001, 0.001, 0.01, 0.1] \\
                        & momentum      & [0, 0.9] \\
                        & aug. strat.   & [full, partial] \\
    \hline
    \multirow{5}*{SCR}      & optim         & [SGD, Adam] \\
                        & weight decay  & [0, 1e-4] \\
                        & lr            & [0.0001, 0.001, 0.01, 0.1] \\
                        & momentum      & [0, 0.9] \\
                        & aug. strat.   & [full, partial] \\
    \hline
    \multirow{1}*{GDumb}    & aug. strat.   & [full, partial] \\
    \hline
    \multirow{6}*{FD-AGD}   & optim         & [SGD, Adam] \\
                        & weight decay  & [0, 1e-4] \\
                        & lr            & [0.0001, 0.0005,  0.001, 0.005, 0.01, 0.05, 0.1] \\
                        & momentum      & [0, 0.9] \\
                        & aug. strat.   & [full, partial] \\
                        & var           & [0.05, 0.5, 1, 2, 3, 4, 5, 10] \\
    \hline
    \multirow{6}*{FD-vMF}   & optim         & [SGD, Adam] \\
                        & weight decay  & [0, 1e-4] \\
                        & lr            & [0.0001,  0.001, 0.01, 0.1] \\
                        & momentum      & [0, 0.9] \\
                        & aug. strat.   & [full] \\
    \hline
    \multirow{6}*{PFC}   & optim         & [SGD, Adam] \\
                    & weight decay  & [0, 1e-4] \\
                    & lr            & [0.0001, 0.0005,  0.001, 0.005, 0.01, 0.05, 0.1] \\
                    & momentum      & [0, 0.9] \\
                    & aug. strat.   & [full, partial] \\
                    & var           & [0.05, 0.5, 1, 2, 3, 4, 5, 10] \\
\end{tabular}
\caption{Hyper-parameters tested for every method on CIFAR100, M=5k, 10 tasks.\label{tab:all_hp}}
\end{table}

\section{Hardware and computation}

For compared methods we trained on 2 RTX A5000 GPUs. Figure \ref{fig:time_consumption} references the training time of each method on CIFAR100 M=5k. Our method can achieve best performance while having a low computational overhead. Notably, the time consumption difference between SCR and FD-AGD/FD-vMF is due to the number of augmentations used for training.

\begin{figure}
    \centering
    \includegraphics{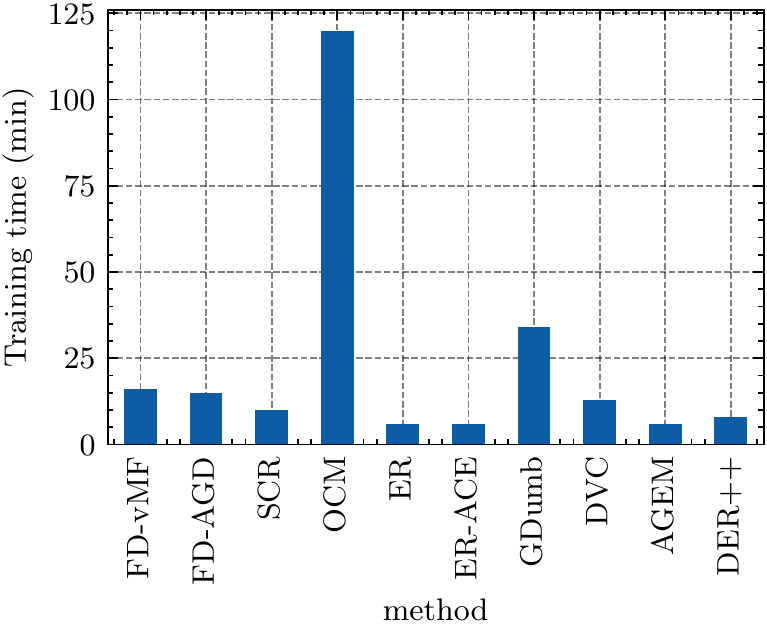}
    \caption{Time consumption in minutes, for every trained methods, on CIFAR100, M=5k, and 10 tasks.}
    \label{fig:time_consumption}
\end{figure}

\section{Additional Experiments}
In this section we gives more detail dataset used in the main paper.

\subsection{Datasets}
As described in the main paper, for the CIL case we experimented on CIFAR10, CIFAR100\cite{Krizhevsky2009LearningML} and Tiny ImageNet~\cite{le_tiny_2015} with blurry and clear task boundaries. In the Appendix, we included partial experiments on CORe50~\cite{lomonaco2017core50}.

\textbf{CIFAR10} contains 50,000 32x32 train images as well as 10,000 test images and is split into 5 tasks containing 2 classes each for a total of 10 distinct classes.

\textbf{CIFAR100} contains 50,000 32x32 train images as well as 10,000 test images and is split into 10 tasks containing 10 classes each for a total of 100 distinct classes.

\textbf{Tiny ImageNet} is a subset of the ILSVRC-2012 classification dataset and contains 100,000 64x64 train images as well as 10,000 test images and is split into 100 tasks containing 2 classes each for a total of 200 distinct classes.

\textbf{CORe50} is a domain incremental dataset designed for continuous object recognition. It is composed of $164,866$ 128×128 RGB images. We experimented in the New Instances setting and used sessions 3,7 and 10 for testing.

\subsection{Experiments in Domain Incremental Learning (DIL) scenario}
To further experiments on various Continual Learning scenarios, we included experiments on CORe50~\cite{lomonaco2017core50} for \textit{ER}, \textit{SCR} and \textit{AGD-FD} and various memory sizes. 

\paragraph{Domain Incremental Learning (DIL)} is another popular Continual Learning scenario where the distribution of the input data shift while keeping the same available classes between tasks. For example, the lighting conditions can changes from one task to the other. One popular dataset for evaluation in DIL is CORe50~\cite{lomonaco2017core50}.

\paragraph{Results}

We report the final average accuracy for considered methods on Table \ref{tab:core50}. It can be observed that our approach is on par with state-of-the art for every memory size. Additionally, the standard deviation of AGD-FD (ours) decreases for larger memory sizes, demonstrating superior stability for large memory size compared with other considered methods.

\begin{table}[!ht]
    \centering
    \begin{tabular}{l|lll}
           \hline
           \multicolumn{1}{c}{Method}          & \multicolumn{1}{c}{M=1k}   &  \multicolumn{1}{c}{M=2k}     &        \multicolumn{1}{c}{M=5k}  \\
           \hline\hline
           \text{ER}       & 32.0±2.6   & 36.6±3.0 & 39.3±2.2 \\
           \text{SCR}      & 42.1±2.0   & \textbf{46.5±3.2} & 48.6±2.3 \\
           \text{AGD-FD}   & \textbf{42.6±2.3}  & 46.0±1.6 & \textbf{50.0±0.9} \\
    \end{tabular}
    \caption{Final average accuracy on CORe50, new instances setting, for varying memory sizes. Mean and standard deviation over 5 runs are reported.}
    \label{tab:core50}
\end{table}

\end{appendices}
% ------------------------------------------------
% \paragraph{Strong focus on new classes} \textbf{Probably moving this + figure \ref{fig:loss_values} into appendix}.
% Another direct effect of fixing means is a considerably higher loss for new classes. Intuitively, when images from unseen classes are projected onto the latent space, their latent representations are most likely far from their affected means. This distance directly translates to large loss value and forces the model to focus on new classes. On figure \ref{fig:loss_values} we can observe the ratio between the loss for new and old classes. This ratio is showed for our method and SCR, and is clearly higher for the former.
% \begin{figure}[!ht]
%     \centering
%     \includegraphics[width=0.35\textwidth]{images/loss_val.png}
%     \caption{Loss ratio between new and old classes on CIFAR100. Task change occurs every 500 batches.\label{fig:loss_values}}
% \end{figure}

%%%%%%%%%%%%%%%%%%%%%%%%%%%%%%%%%%%%%%%%%%%%%%%%%%%%%%%%%%%%

\end{document}